\documentclass{article}

\usepackage{arxiv}

\usepackage[utf8]{inputenc} % allow utf-8 input
\usepackage[T1]{fontenc}    % use 8-bit T1 fonts
\usepackage{hyperref}       % hyperlinks
\usepackage{url}            % simple URL typesetting
\usepackage{booktabs}       % professional-quality tables
\usepackage{amsfonts}       % blackboard math symbols
\usepackage{nicefrac}       % compact symbols for 1/2, etc.
\usepackage{microtype}      % microtypography
\usepackage{xcolor}         % colors
\usepackage{soul}         % colors

\usepackage{natbib}
\usepackage{microtype}
\usepackage{graphicx}
\usepackage{authblk}
\usepackage{hyperref}
\hypersetup{
	colorlinks=true,
	linkcolor=blue,
	filecolor=blue,      
	urlcolor=blue,
	citecolor=blue,
	pdftitle={Overleaf Example},
	pdfpagemode=FullScreen,
}
\usepackage{subfigure}
\usepackage{booktabs} % for professional tables

\usepackage{hyperref}
\usepackage{dsfont} % for indicator
\usepackage{algorithm}
\usepackage{amsmath}
\usepackage{amssymb}
\usepackage{mathtools}
\usepackage{amsthm}

\newcommand{\IP}[2]{\left\langle #1,#2 \right\rangle}

\newcommand{\reals}{\mathbb R}

\DeclareMathOperator{\E}{\mathbb E}
\DeclareMathOperator{\argmin}{arg\,min}
\DeclareMathOperator{\diag}{diag}

\usepackage[capitalize,noabbrev]{cleveref}

\theoremstyle{plain}
\newtheorem{theorem}{Theorem}[section]

\newtheorem{lemma}[theorem]{Lemma}

\theoremstyle{definition}

\theoremstyle{remark}

\usepackage{xspace}
\usepackage{todonotes}
\newcommand{\trace}[1]{\mbox{tr}\left(#1\right)}

 \newcommand{\X}{\mathcal{X}}
 
\newcommand{\F}{\mathcal{F}}

\newcommand{\real}{\mathbb{R}}

\newcommand{\DD}{\mathcal{D}}
\newcommand{\OO}{\mathcal{O}}
\newcommand{\tOO}{\wt{\OO}}
\newcommand{\II}[1]{\mathds{1}\left\{#1\right\}}
\newcommand{\PP}[1]{\mathbb{P}\left[#1\right]}

\newcommand{\EE}[1]{\mathbb{E}\left[#1\right]}

\newcommand{\EEw}[2]{\mathbb{E}_{#2}\left[#1\right]}

\newcommand{\PPt}[1]{\mathbb{P}_t\left[#1\right]}

\newcommand{\EEt}[1]{\mathbb{E}_t\left[#1\right]}
\newcommand{\EEs}[1]{\mathbb{E}_s\left[#1\right]}

\newcommand{\PPcc}[2]{\mathbb{P}\left[\left.#1\right|#2\right]}

\newcommand{\EEc}[2]{\mathbb{E}\left[#1\left|#2\right.\right]}
\newcommand{\EEcc}[2]{\mathbb{E}\left[\left.#1\right|#2\right]}

\newcommand{\EEcct}[2]{\mathbb{E}_t\left[\left.#1\right|#2\right]}

\def\argmin{\mathop{\mbox{ arg\,min}}}

\newcommand{\siprod}[2]{\langle#1,#2\rangle}
\newcommand{\iprod}[2]{\left\langle#1,#2\right\rangle}

\newcommand{\norm}[1]{\left\|#1\right\|}

\newcommand{\pa}[1]{\left(#1\right)}

\newcommand{\wh}{\widehat}
\newcommand{\wt}{\widetilde}

\newcommand{\loss}{\ell}

\newcommand{\lambdamin}{\lambda_{\min}}

\newcommand{\hR}{\wh{R}}

\newcommand{\htheta}{\wh{\theta}}

\newcommand{\hTheta}{\wh{\Theta}}
\newcommand{\ttheta}{\wt{\theta}}

\newcommand{\tQ}{\wt{Q}}

\newcommand{\tS}{\widetilde{\Sigma}}
\newcommand{\tZ}{\widetilde{Z}}

\newcommand{\transpose}{^\mathsf{\scriptscriptstyle T}}

\usepackage{todonotes}
\definecolor{PalePurp}{rgb}{0.66,0.57,0.66}

\newcommand{\hL}{\wh{L}}

\newcommand{\clinexp}{\textsc{ContextEW}\xspace}

\newcommand{\linexp}{\textsc{LinExp3}\xspace}

\title{First- and Second-Order Bounds for Adversarial Linear Contextual Bandits}

\author[1]{Julia Olkhovskaya}
\author[2]{Jack Mayo}
\author[2]{Tim van Erven}
\author[3]{Gergely Neu}
\author[3]{Chen-Yu Wei}
\affil[1]{Department of Mathematics, Vrije Universiteit Amsterdam, Amsterdam, The Netherlands}
\affil[2]{Korteweg-de Vries Institute for Mathematics, University of Amsterdam, Amsterdam, The Netherlands}
  \affil[3]{AI group, DTIC, Universitat Pompeu Fabra, Barcelona, Spain}
 \affil[4]{MIT Institute for Data, Systems, and Society, Massachusetts Institute of Technology, Cambridge, MA, USA}

\begin{document}
\maketitle

\begin{abstract}
  We consider the adversarial linear contextual bandit setting, which
  allows for the loss functions associated with each of $K$ arms to change
  over time without restriction. Assuming the $d$-dimensional contexts are
  drawn from a fixed known distribution, the worst-case expected regret
  over the course of $T$ rounds is known to scale as $\tilde O(\sqrt{Kd
  	T})$. Under the additional assumption that the density of the contexts
  is log-concave, we obtain a second-order bound of order $\tilde
  O(K\sqrt{d V_T})$ in terms of the cumulative second moment of the
  learner's losses $V_T$, and a closely related first-order bound of order
  $\tilde O(K\sqrt{d L_T^*})$ in terms of the cumulative loss of the best
  policy $L_T^*$. Since $V_T$ or $L_T^*$ may be significantly smaller than
  $T$, these improve over the worst-case regret whenever the environment
  is relatively benign. Our results are obtained using a truncated version
  of the continuous exponential weights algorithm over the probability
  simplex, which we analyse by exploiting a novel connection to the linear
  bandit setting without contexts.
\end{abstract}

\section{Introduction}
The contextual bandit problem is a generalization of the multi-armed
bandit setting in which a learner observes relevant contextual
information before choosing an arm. The goal of the learner is to
minimize the excess cumulative loss of the chosen arms compared to the
best fixed policy for mapping contexts to arms. This framework addresses
a broad range of important real-world problems like sequential treatment
allocation \citep{Tewari2017}, online recommendation
\citep{BLLRS11} or online advertising \citep{LCLS10}, and
is actively used in practice \citep{1606.03966}. Numerous variants of the
setting have been studied, which differ in the assumptions they make
about the losses and the contexts. In this paper, we focus on the recently introduced setting of \citet{2020NO} where 
the contexts are finite-dimensional i.i.d.~random vectors, and the losses are time-varying linear
functions of the context that may potentially be generated by an
adversary. In this setting, the worst-case rate for the expected regret
is known to be $\tilde O(\sqrt{T})$ for time horizon $T$ \citep{2020NO}.

Our main contribution is to replace the worst-case rate by adaptive
bounds. Specifically, we obtain a bound of $\tilde O(\sqrt{V_T})$ in
terms of a quadratic measure of variance $V_T$ for the losses of the
algorithm, and a bound of $\tilde O(\sqrt{L_T^*})$, where $L_T^*$ is the
cumulative loss incurred by the optimal policy. Such bounds in terms of
$L_T^*$ or $V_T$ are generally referred to as \emph{first-order} and
\emph{second-order bounds}, respectively, and have been extensively
studied in the bandit literature. They
% Tim: remove this, because we have a worse dependence on $K$
%are a strict improvement over the standard worst-case bounds, and
can lead to much stronger guarantees in the often realistic case when
$T$ is large, but the losses vary little or when there exists a policy
with very low cumulative loss.

Worst-case guarantees in terms of $T$ have first been proved for the
contextual bandit problem with finite policy classes by
\citet{AuerCFS02}, with further improvements by \citet{BLLRS11}. These
methods can deal with adversarial losses and contexts, but only work for
finite policy classes and have run-time scaling linearly with the size of
the class---which is generally unacceptable in practice.  This latter
challenge has been addressed by a line of work culminating in
\citet{pmlr-v32-agarwalb14}, which only requires access to an
optimization oracle over the policy class. Their results, however,
remain restricted to i.i.d.~contexts and losses. An alternative line of
work has been initiated by \citet{Aue02,CLRS11,APSz11}, who studied the
special case of i.i.d.~\emph{linear} loss functions with changing
decision sets. The case of i.i.d.~contexts and adversarial linear losses
has first been studied by \citet{2020NO}.

Improvements of worst-case guarantees of order $\sqrt{T}$ to first-order bounds scaling with $\sqrt{L_T^*}$ have been 
known for a variety of bandit settings since the works of \citet{S05,AAGO06}, and \citet{Neu15}. Regarding contextual 
bandits, the COLT 2017 open problem of
\citet{AKLLS17} asks for efficient algorithms that achieve first-order
bounds for large, but finite, policy classes, either when both contexts
and losses are i.i.d.\ or when both are fully adversarial. First to
answer the open problem were \citet{ABL18}, who obtained an optimal
first-order regret guarantee for adversarial losses and contexts, but
with an algorithm that is inefficient for large policy classes.
\citet{Foster202118907} provide the first efficient algorithm for the
non-adversarial setting where the loss function is fixed over time and
one has access to an oracle that can solve various optimization tasks
over the policy class. We improve on these works in terms of the
computational efficiency of our algorithm and by allowing the loss
function to vary adversarially over time, although we do rely on the
extra assumption that the loss functions are linear.

Another relevant framework is the adversarial linear bandit setting
(without contexts), where there also exist adaptive results
\citep{pmlr-v99-bubeck19b,NEURIPS2020_b2ea5e97, NEURIPS2020_15bb63b2}.
While conceptually related, an important distinction is that the linear
bandit setting assumes a fixed decision set, whereas reducing the linear
contextual bandit problem to a linear bandit problem requires the use of
decision sets that change as a function of the contexts.

\paragraph{Main Contributions.}
We consider a $K$-armed linear contextual bandit problem with
$d$-dimensional contexts over $T$ rounds. The contexts are assumed to be
drawn i.i.d., but the linear loss functions mapping contexts to losses
for the arms are chosen by an adaptive adversary. The aim of the learner
is to minimize their regret, which is the gap between the expected
cumulative loss of the learner and the expected cumulative loss of the
best fixed policy $\pi_T^*$ chosen in full knowledge of the sequence of
losses. In this setting, $\pi_T^*$ is known to be a linear classifier,
i.e.\ it chooses the arm with smallest predicted loss, where the
predictions are fixed linear functions of the context (see
Section~\ref{sec:preliminaries}). The goal is therefore to compete with
all linear classifiers. We first obtain the following second-order bound
on the expected regret
\begin{equation}\label{eqn:intro_second_order}
  R_T = \tilde{O}\Big(K\sqrt{d V_T} \Big),
\end{equation}
where $V_T$ is defined in \eqref{eqn:VT-def} as a measure of the
cumulative second moments of the losses for the arms played by the
algorithm. Following \citet{NEURIPS2020_15bb63b2}, we allow these
moments to be centered around optimistic estimates that can further
improve the bound when available or can simply be set to zero when they
are not. We further obtain a first order bound of the form
\begin{equation}\label{eqn:intro_first_order}
  R_T(\pi_T^*) = \tilde{O}\Big(
    K \sqrt{d L_T^*}
  \Big).
\end{equation}
The second-order bound is obtained using a truncated version of the
continuous exponential weights algorithm over the probability simplex,
similar to the algorithm for linear non-contextual bandits of
\citet{NEURIPS2020_15bb63b2}, and the first-order bound may be obtained
as a corollary. As discussed in Section~\ref{sec:computation}, the
computational complexity of this method is dominated by two steps that
together require $\tilde O(K^5) + (d/\epsilon)^{O(1)}$ per round for
approximation up to precision $\epsilon > 0$, which is computationally
feasible for moderate $K$ and $\epsilon$.
Both results are not strict improvements on the worst-case rate of
$\tilde O(\sqrt{KdT})$ by \citet{2020NO}: first, they have a slightly
worse dependence on $K$. We consider this a price worth paying for
the first adaptive bounds in this setting. Second,
they require the extra assumption that the distribution of the contexts
is \emph{log-concave}. Although log-concavity is weaker than assuming
the contexts follow e.g.\ (truncated) Gaussian distributions, we
conjecture that it may not
be necessary to obtain a computationally efficient algorithm. This
conjecture is based on the observation that there exists in fact an easy way to
obtain at least the first-order bound \eqref{eqn:intro_first_order}
without the log-concavity assumption, but with an algorithm that has no
hope of being efficiently implemented. As described in
Section~\ref{sec:myga}, this is possible by running the MYGA algorithm
\citep{ABL18} on $O(\frac{T}{dK^2})^{Kd}$ experts that cover the set of
linear classifiers to sufficient precision. The run-time of this
approach is prohibitive, because it scales linearly with the number of
experts, which is a large polynomial in~$T$.

\paragraph{Techniques.}
The LinExp3 method of \citet{2020NO} is based on an
adaptation of the classic Exp3 algorithm for regular multi-armed bandits
\citep{auer2002finite}. A natural approach would therefore be to replace
the Exp3 component in LinExp3 by a method with first-order guarantees
for the multi-armed bandit setting, but, as discussed in
Section~\ref{sec:exp3_discussion}, this leads to
difficulties controlling the variance.
%We have tried this for the Green
%algorithm \citep{AAGO06} and for Exp3Light \citep{S05}, but, as discussed
%in Section~\ref{sec:discussion}, both approaches ran into problems
%controlling the bias and variance.\todo{Update this after writing
%discussion} Another approach, which we have not explored, might be to
%build on the more recent works by \citet{WL18} and
%\citet{NEURIPS2020_b2ea5e97}. \todoG{\textbf{Grg:} I'm not sure if this is the proper place to mention failed 
%attempts, unless there is a clear way to explain what goes wrong.}
Instead of building on Exp3, we therefore follow the
perhaps surprising approach of building our algorithm on \emph{continuous exponential weights} over the probability 
simplex \citep{HEK18}. In particular, our approach is based on a combination of the recently proposed techniques of 
\citet{NEURIPS2020_15bb63b2} for linear bandits with tools designed by \citet{2020NO} to deal with the contextual case. 
%In particular, we build on a technical lemma from \citet{2020NO} which replaces the time-varying contexts by a fixed
%context in expectation, and execute the algorithm of \citet{NEURIPS2020_15bb63b2} on the resulting convex decision set. 
%\todoG{\textbf{Grg:} Perhaps too much detail for this part of the paper.}

\paragraph{Outline.}
The rest of the paper is organized as follows. After describing the
setting in the next section, we state a formal version of the simple
first-order bound  that can be obtained using the MYGA algorithm
(Theorem~\ref{thm:myga_reduction}). This is followed by
Section~\ref{sec:results}, which states our main results corresponding
to the regret bounds in Equations~\ref{eqn:intro_second_order} and
\ref{eqn:intro_first_order}. Section~\ref{sec:analysis} then gives a
high-level overview of the proofs, with pointers provided to the details
in the appendix. Finally, Section~\ref{sec:discussion} concludes with
discussion.

\section{Preliminaries}\label{sec:preliminaries}
\paragraph{Notation}

Let $\Delta^{K}=\{w \in \real^K | w_1 \geq 0,\ldots, w_K \geq 0,
\sum_{a=1}^{K} w_a = 1 \}$ denote the $(K-1)$-dimensional probability
simplex. For any positive semi-definite matrix $M\in \real^{d\times d}$,
$\norm{v}_{M} = \sqrt{v^{\transpose}M v}$ denotes the corresponding
Mahalanobis norm, and for any positive integer $n$, we abbreviate $[n] =
\{1,\ldots,n\}$.

\subsection{Setting}
We consider the setting of \citep{2020NO}, in which there is an
interaction between a learner and an unknown environment. This
interaction proceeds in rounds indexed by $t\in [T]$, such that for each
$t$: 
\begin{itemize}
\item[1.] The environment commits to $[K]$ parameter vectors $\theta_{t,1},\dots,\theta_{t,K} \in \real^{d}$ without revealing any to the learner.
\item[2.] A context vector $X_{t} \in \real^{d}$ is drawn i.i.d. from some fixed distribution $\mathcal{D}$ according to $X_{t}\sim \mathcal{D}$, and revealed to the learner.
\item[3.] The learner commits to an action $A_{t}\in [K]$, and incurs
the loss $\ell_t(X_{t},A_{t})$, where $\ell_t(X,a)
=\iprod{X}{\theta_{t,a}}$.
\end{itemize}

The environment is allowed to randomize its choices of $\theta_{t,a}$.
These must be independent from the context $X_t$ in round $t$, but they
may depend on previous contexts $X_s$ and actions $A_s$ for $s < t$.

We write $\pi_t(a | X_t)$ for the policy of the learner in round $t$
conditional on observing context $X_t$, so that $A_t \sim \pi_t(X_t)$, and we use the following notation for the expected cumulative
losses of the algorithm and policy $\pi$, respectively:
\begin{align*}
  L_T =
  \mathbb{E}\left[\sum_{t=1}^{T} \ell_t(X_{t},A_{t})\right],
  L_T^{\pi} = 
  \mathbb{E}\left[\sum_{t=1}^{T} \ell_t(X_{t},\pi(X_{t}))\right].
\end{align*}
Let $\Pi$ be the set of all all stationary deterministic policies
$\pi : \real^{d}\rightarrow [K]$, we define the optimal policy $\pi^*$ as $\pi^*= \argmin_{\pi \in \Pi} L^{\pi}_T$.
Then the learner's goal is to compete with policy $\pi^*$, as measured by the expected regret:
\begin{align*}
R_{T}&=L_T - L_T^{\pi^*}
=\mathbb{E}\left[\sum_{t=1}^{T}\iprod{X_{t}}{\theta_{t,A_{t}}-\theta_{t,\pi^*(X_{t})}}\right],
\end{align*}
where the expectation is taken over each $X_{t}\sim\mathcal{D}$, and any
randomness applied by the learner or environment in their respective
choices. 
Using the linearity of the loss functions it can be shown that
the optimal policy is always a linear classifier \citep{2020NO}:
\[
  \pi^*_T(x) = \argmin_a \iprod{x}{\sum_{t=1}^T \E[\theta_{t,a}]}.
\]
We may therefore restrict attention to competing with policies of the
form
\begin{equation}\label{eqn:linear_classifier}
  \pi_\beta(x) = \argmin_a \iprod{x}{\beta_a}
  \qquad
  (\beta \in \reals^{K\times d}).
\end{equation}
For deriving our technical results, it will be useful to define
the filtration $\F_t = \sigma(\{X_s, A_s: s \le t\})$, and the notations
$\EEt{\cdot} = \EEcc{\cdot}{\F_{t-1}}$ and $\PPt{\cdot} =
\PPcc{\cdot}{\F_{t-1}}$.

\paragraph{Assumptions}

Following \citet{2020NO}, we assume that $\|X_t\| \leq \sigma$,
$\|\theta_{t,a}\| \leq R$ and $\ell_{t}(x,a)\in[-1,1]$ almost surely. In
addition, the covariance matrix $\Sigma=\mathbb{E}[XX^{\transpose}]$ of
the context distribution is assumed to be positive definite, with
smallest eigenvalue $\lambda_{\mathrm{min}}(\Sigma)>0$.

\subsection{An Inefficient Algorithm}\label{sec:myga} 

A first order bound for our problem can be obtained by instantiating the
MYGA algorithm of \citet{ABL18} for a set of
$\Theta(\frac{T}{K^2d})^{Kd}$ experts that cover the parameter space of
policies of the form \eqref{eqn:linear_classifier}, which is guaranteed
to contain the optimal policy $\pi^*_T$:
\begin{theorem}\label{thm:myga_reduction}
  Suppose that $0\le \ell_{t}(a,X_t) \le 1$ almost surely for all
  $a\in[K]$. Then, by instantiating MYGA with
  $\Theta(\frac{T}{K^2d})^{Kd}$ experts, it obtains the following
  first-order bound for the adversarial linear contextual bandit
  problem:
    \begin{equation}
            R_{T}= O\left( K\sqrt{dL_{T}^{*}\log T} +K^{2}d \log T \right).
    \end{equation}
\end{theorem}
Although this provides a quick way to see that first-order bounds are
possible, the resulting algorithm is completely impractical, because its
run-time is proportional to the number of experts, which grows as a
large polynomial in $T$. The proof, including a more detailed
description of the experts, can be found in Appendix~\ref{myga_proof}.

\section{First- and Second-Order Bounds}\label{sec:results}

%\subsection{Result 1}
In this section we present an algorithm using a novel adaptation of methods developed for the adversarial linear bandit to be suitable for use in the adversarial linear contextual bandit setting. The method proposed is based on a form of continuous exponential weights that has been shown to lead to
a first-order bound in the former \citep{NEURIPS2020_15bb63b2}.
The algorithm allows for optimistic estimates $m_{t,a} \in \real^{d}$
for the environment's choices $\theta_{t,a}$, which can always be set to
$0$ when they are not available. We show two types of guarantees. First,
in Theorem~\ref{var_theorem}, we obtain a second-order regret bound in
terms of the cumulative squared error of the estimates $m_{t,a}$:
\begin{equation}\label{eqn:VT-def}
  V_T = \EE{  \sum_{s=1}^{T} \iprod{X_s}{\theta_{s,A_s} - m_{s,A_s}}^2}.
\end{equation}
Taking $m_{t,a} = 0$, this provides a second-order regret bound in terms
of the squared losses. Alternatively, $m_{t,a}$ may be estimated using
an online regression algorithm, as described by
\citet{NEURIPS2020_15bb63b2}. As our second result, we show in
Theorem~\ref{first_order} that a first-order bound can be derived for
the same algorithm with a different choice of hyperparameters and the
assumption that the losses are non-negative.

\begin{algorithm}[h]
	\caption{\clinexp}
	\label{alg:linexp3}
        \textbf{Parameters:}  $\gamma>0$, $\eta_1 \geq \ldots \geq
        \eta_T>0$, $m_1,\ldots,m_T$\\
	%\textbf{Initialization:} Set $\htheta_{0,a} = 0$ for all $a\in[K]$.
        % Tim: Removed initialization, because it is not used
	\textbf{For} $t = 1, \dots, T$:
	\begin{enumerate}
		\item Observe $X_t$.
		\item \textbf{Repeat:}\\
		\quad Pick $Q_t$ from the distribution $p_t$ defined in
               \eqref{exp3}, \textbf{until  }
		\begin{equation}\label{truncation_event}
			\sum_{a=1}^K \norm{Q_{t,a} X_t}^2_{	\Sigma_{t,a}^{-1}} \le dK\gamma^2,
		\end{equation} 
               where $\Sigma_{t,a}$ is defined in
               \eqref{eqn:Sigma_def}.
               \item Set $\tQ_t = Q_t$ equal to the last sample of
               $Q_t$, which caused the loop to exit, and choose an arm
               according to
              $
                 A_t \sim \tQ_{t}.
               $
% 		\[
% 		\pi_{t}\left(a\middle|X_t\right) = \tQ_{t,a}.
% 		\]
		\item Observe the loss $\loss_t(X_t,A_t)$ and estimate
               $\htheta_{t,a}$ for all $a$ according to
               \eqref{eqn:theta_est}.
	\end{enumerate}
\end{algorithm}

\subsection{Algorithm Description}

Our full algorithm is shown in Algorithm~\ref{alg:linexp3}. As it is an
adaptation of continuous exponential weights for the contextual bandits
setting, we refer to it as \clinexp. It runs a two-stage sampling
procedure: after observing context $X_t$, the first stage of the
algorithm samples a random policy $\tQ_t \in \Delta^K$, and then the
second stage consists of drawing an arm $A_t$ randomly from $\tQ_t$. The
distribution of $\tQ_t$ is constructed as follows: first we sample a
different policy $Q_t$ from the exponential weights distribution over
the probability simplex with density proportional to
\begin{equation}\label{weights}
	w_t(q|X_t) = \exp\pa{- \eta_t \sum_{a=1}^K q_a\iprod{X_t}{\sum_{s=1}^{t-1} \htheta_{s,a} }}.
\end{equation}
The sum $\sum_{a=1}^K q_a\iprod{X_t}{\sum_{s=1}^{t-1} \htheta_{s,a} }$
estimates the cumulative loss that the policy $q$ would have incurred if
it had been played in all previous rounds. It relies on estimates
$\htheta_{s,a}$ of the loss vectors $\theta_{s,a}$, which will be
defined below, and a time-varying learning rate $\eta_t > 0$, which is
hyperparameter of the algorithm. The normalized density function
corresponding to the weights in \eqref{weights} is:
\begin{equation}\label{exp3}
	p_t(q|X_t) = \frac{w_t(q|X_t) }{\int_{\Delta^K } w_t(q|X_t) dq}.
\end{equation}
Following \citet{ NEURIPS2020_15bb63b2}, we then introduce a rejection
sampling step \eqref{truncation_event} to reduce the variance, which is
based on the following covariance matrices $\Sigma_{t,a}$ corresponding to
$Q_t$:
\begin{equation}\label{eqn:Sigma_def}
  \Sigma_{t,a} = \EEt{Q_{t,a}^2 X_t X_t\transpose},
\end{equation}
so that $\tQ_t$ ends up being sampled according to the following
truncated exponential weights density:
\begin{equation}\label{truncation}
	\tilde p_t(q|X_t) = \frac{p_t(q|X_t) \II{ \sum_{a=1}^K\norm{q_a
        X_t
        }^2_{\Sigma_{t,a}^{-1}}  \le dK\gamma^2} }{ P_t\Big(
        \sum_{a=1}^K\norm{q_a X_t }^2_{\Sigma_{t,a}^{-1}}   \le
        dK\gamma^2|X_t\Big)},
\end{equation}
with truncation level hyperparameter $\gamma>0$. We will show that all
$\Sigma_{t,a}$ are invertible, as are their analogues in which $Q_t$ is
replaced by $\tQ_t$:
\begin{equation}\label{eqn:tSigma_def}
\tS_{t,a} = \EEt{\tQ_{t,a}^2 X_t X_t\transpose}.
\end{equation}
It remains to specify our estimators for $\theta_{t,a}$, which are
defined as follows:
\begin{equation}\label{eqn:theta_est}
  \htheta_{t,a} = 
    m_{t,a} +
    \tQ_{t,a}\tS_{t,a}^{-1} X_t\pa{\siprod{X_t}{\theta_{t,a}}
    - \siprod{X_t}{m_{t,a}} }\II{A_t = a}.
\end{equation}
These estimates can be shown to be unbiased: 
\begin{align*}
	\E_t\Big[\htheta_{t,a}\Big] 
        &=  m_{t,a} + \tS_{t,a}^{-1} \EEt{\tQ_{t,a}  X_t
        X_t\transpose \II{A_t = a}}(\theta_{t,a}- m_{t,a}) \\
        &=  m_{t,a} + \tS_{t,a}^{-1} \EEt{\tQ_{t,a}^2   X_t
        X_t\transpose}(\theta_{t,a}- m_{t,a})
        =  \theta_{t,a}.
\end{align*}
 
\subsection{Results}
 
We instantiate \clinexp{} with adaptive learning rates $\eta_t$.
For our second-order result, these are defined in terms of the empirical
counterpart to $V_t$:
$
  \widehat{V}_t = \sum_{s=1}^{t} \iprod{X_s}{\theta_{s,A_s} -
  m_{s,A_s}}^2,
$
and we abbreviate
$
  G_t = 8 \sqrt{\widehat{V}_{t-1}  \ln(2T^2) + 144\ln^2 T  }  + 176\ln T.
$
Then we set
\begin{equation}\label{eqn:eta_second_order}
  \eta_t  =   (100dK\gamma^2 + d(\widehat{V}_{t-1} + 1 +
  G_{t-1}))^{-1/2}.
\end{equation}
This leads to the following second-order bound:
\begin{theorem}[Second-Order]\label{var_theorem}
      Suppose $\mathcal{D}$ has a log-concave density.
       Then, for $\gamma = 4 \log(10 d KT)$, $\eta_t$ as in
       \eqref{eqn:eta_second_order} and any $\F_{t-1}$-measurable
       estimates $m_t$,  the expected regret of \clinexp is at most
       $R_T = \widetilde{O} (K\sqrt{d V_T})$.
\end{theorem}

To tune $\eta_t$ adaptively for our first-order bound, we define it
using the algorithm's empirical cumulative loss
$
  \widehat{L}_t = \sum_{s=1}^{t} \ell_t(X_s, A_s),
$
which acts as a self-confident empirical estimate of $L_T^*$. We further
abbreviate
\begin{equation}\label{learn_rate_1st}
  H_t = 8 \sqrt{2\hL_t \ln T + 40\ln^2T  }  + 72\ln T,
\end{equation}
and then set
\begin{equation}\label{eqn:eta_first_order}
  \eta_t  =   (100d\gamma^2 + dK(\hL_{t-1} + 1 + H_{t-1}))^{-1/2}.
\end{equation}
This leads to the following first-order bound:
 \begin{theorem}[First-Order]\label{first_order}
        Suppose that $\mathcal{D}$ has a log-concave density and that
        $0\le \ell_{t}(a,X_t) \le 1$ almost surely for all $a\in[K]$.
        Then, for $\gamma = 4 \log(10 d KT)$, $\eta_t$ as in
        \eqref{eqn:eta_first_order} and $m_t = 0$, the expected regret
        of \clinexp is at most $R_T = \widetilde{O}(K\sqrt{d L^*_T})$.
\end{theorem}

\subsection{Computational Efficiency}\label{sec:computation}

The two computational bottlenecks in the algorithm are the cost of
sampling from the output distribution $p_{t}(q|X_t)$ and computation of
the covariance matrices $\Sigma_{t,a}$ in each round.

Due to the log-linearity of our method, there exists several practical methods of sampling. As mentioned in \cite{NEURIPS2020_15bb63b2}, one can employ the methods of \cite{LV07}, which was shown in \cite{Lovasz2} to enjoy a bound of $O(K^4\mathrm{log}(1/\epsilon))$ (where $\epsilon$ is a bound on the total variation distance between the output distribution and the target), but this still requires knowledge of a density dominating the target distribution on all but a set with total starting mass $\leq \epsilon/2$. In \cite{rakhlin}, a method is developed for general log-concave distributions which, specialized to log-linear distributions (and without additional assumptions on the initial distribution) yields an $O(K^{3}\nu^2+\mathrm{log}(1/\epsilon))$ method when the geometry admits a $\nu$-self concordant barrier. Since there always exists a $K$-self-concordant barrier for a $K$-dimensional convex body, and thus the running time of this method for our problem is $O(K^{5}+\mathrm{log}T)$ up to a precision $\epsilon \sim \frac{1}{T^{\beta}}$ for some $\beta >0$.
As referred to in \cite{NEURIPS2020_15bb63b2}, the covariance matrix
$\Sigma_{t,a}$ is computable in $\OO((d/\epsilon)^{O(1)})$ sampling
steps drawing upon the results of \cite{LV07}.

%\subsection{Result 2}
%\input{result_2}

\section{Analysis}\label{sec:analysis}
In this section we provide the analysis of \clinexp from which Theorems~\ref*{var_theorem} and~\ref*{first_order} follow. 
Throughout the analysis, we will be extensively using the following property of log-concave distributions:
\begin{lemma}\label{lem:logconc}
	If $x$ follows a log-concave distribution $p$ over $\real^d$ and $\EE{x x\transpose}\preccurlyeq I$, we have, for any $\alpha \ge 0:$
	\begin{equation}\label{logcocavity}
		\PP{\norm{x}^2_2 \ge d\alpha^2} \le d\exp(1-\alpha).
	\end{equation} 
\end{lemma}
This result was proven in Lemma~1 in \cite{NEURIPS2020_15bb63b2}, and also follows from Lemma 5.7 in \cite{LV07}.

First, we need to introduce some notation which will be useful for the reduction to the linear bandit setting and for the accompanying proofs. 
We denote $z_a(q,x) =  q_a x$  and $z(q,x) = (z_1(q,x), \dots,
z_K(q,x))\transpose$. We also define $\Sigma_t = \diag_{a\in[K]}(\Sigma_{a,t})$ as a block diagonal arrangement of the covariance matrices per arm. 
Using this notation, the distribution of the sampling algorithm (\ref{truncation}) may be rewritten as
\begin{equation}
	\tilde p_t(q|x) = \frac{p_t(q|x) \II{ \norm{z(q,x) }^2_{\Sigma_t^{-1}}  \le dK\gamma^2} }{ \PPt{ \norm{z(q,x) }^2_{\Sigma_t^{-1}}  \le dK\gamma^2}}.
\end{equation}
% Let  $\tS_t = \int_{\mathcal{X}} p(x) \int_{\Omega} \tilde p_t(q|x) z(q,x) z(q,x)\transpose dq dx  $, where $\tilde p_t(q|x)$ is the density induced by the distribution of the sampling algorithm, which is defined later. 
Let $\tQ_t(x) \sim \tilde p_t(q|x) $, $Q_t(x) \sim  p_t(q|x) $ and
$\tilde Z_t(x) = z(\tQ_t(x), x)$, $Z_t(x) = z(Q_t(x), x)$, $Z^*(x) = z(\pi^*(x), x)$. 
And we denote the aggregated loss parameter $\theta_t = (\theta_1, \dots, \theta_K)\transpose$ and its estimate $\htheta_t = (\htheta_1, \dots, \htheta_K)\transpose$.
%We prove that $\tS_t$ and  $\tS_{t,a}$ is full rank later on the analysis.   
Then we can express the regret as follows:
\begin{align}\label{regret_linbandit}
	R_{T} = \mathbb{E}\left[\sum_{t=1}^{T}\ell_t(X_{t},A_{t})-\ell_t(X_{t},\pi^*(X_{t}))\right]
	= \EE{ \sum_{t=1}^T \iprod{\tilde Z_t(X_t)-  Z^*(X_t)}{\theta_t}  }.
\end{align}
The crucial observation is that the log-concavity of the distribution of $Z_t(X_t)$ follows from that of the distribution of $X_t$:
\begin{lemma}\label{logconcavitylemma}
	Suppose $z(q,x)=\sum_{a}q_{a}\varphi(x,a)$ for $\varphi(x,a)=(\bar{0}^{\intercal},\dots,x^{\intercal},\cdots)$ such that $x$ is on the $da$'th co-ordinate and $Q(x) \sim p_t(\cdot|x)$ for $p_t(\cdot |x)$ defined in (\ref{exp3}). If $X \sim p_{X}(\cdot)$ and $p_{X}(\cdot)$ is log-concave and $Z(x)=z(Q_{t}(x),x)$, then $Z(X)$ also follows a log-concave distribution.
\end{lemma}
The proof of  this result is a rather straightforward computation of the
density of $Z_t(X_t)$ and can be found in Appendix~\ref{ap:proof}. To proceed, we  write regret as a sum of two terms
\begin{align}\label{regret_dec}
	R_t = \EE{ \sum_{t=1}^T \iprod{\tZ_t(X_t) - Z_t(X_t) }{\theta_t}} + \EE{   \sum_{t=1}^T \iprod{ Z_t(X_t) - Z^*(X_t) }{\theta_t} }.
\end{align}
Having shown that $Z_t(X_t)$ is log-concave, and since the log-concavity is preserved under linear transformations, for $y = \Sigma_t^{-1/2\transpose} Z_t(X_t)$ we can see that $\EE{y y\transpose} = I$, and thus by Lemma~\ref{lem:logconc} it immediately follows that the probability that (\ref{truncation_event}) is not satisfied is small for a proposed choice of $\gamma = 4 \log(10 d KT)$:
\begin{align*}
	 \PPt{\norm{
			Z_t(X_t)}_{\Sigma_t^{-1}}^2> dK\gamma^2} \le dK \exp(1-\gamma)
		 \le 3dK \exp(-\gamma) \le \frac{1}{6T^2}.
\end{align*} 
Using this observation, we show that the first term of (\ref{regret_dec}) is just $\OO(1)$, which is formally proved in Lemma~\ref{logconcconcentration} in the appendix.  

To control the second term of the regret decomposition (\ref{regret_dec}), consider the reduction of the contextual bandit problem to a combination of auxiliary online learning problems that are defined separately for each context, as proposed in \cite{2020NO}, Lemma 3. More details and a full proof can be found  in Appendix~\ref{ap:proof}. 
\begin{lemma}\label{ghost} Let $\pi^*$ be any fixed stochastic policy and let $X_0\sim\DD$ be a sample from the context distribution independent from $\F_T$. Suppose that $p_t \in \F_{t-1}$, such that $p_t(\cdot|x)$ is a probability density with respect to Lebesgue measure with support $\Delta^K$ and let $Q_t(x)\sim p_t(\cdot|x)$. Then,	\begin{align}\label{eq:ghost}
		&\EEt{  \iprod{Z_t(X_t) -  Z^*(X_t)}{\theta_t}  } = \EEt{   \iprod{Z_t(X_0) -  Z^*(X_0)}{\htheta_t}}.%	\end{align}
%	\begin{align}\label{eq:ghost}
%	\EE{ \sum_{t=1}^T \iprod{z(Q_t(X_t), X_t) -  z(\pi^*(X_t), X_t)}{\theta_t}  } = \EE{ \sum_{t=1}^T  \iprod{z(Q_t(X_0), X_0) -  z(\pi^*(X_0), X_0)}{\htheta_t}}.
    \end{align}
\end{lemma}
To see why this would be useful further in the proof, we interpret the right-hand side of (\ref{eq:ghost}) as follows. Consider the online learning problem for a fixed $x$ with the decision set to be $\Delta^K$ and losses $\ell_{t}(x,q) = \siprod{z(q, x)}{\htheta_t}$ and consider running a version of a contextual bandit problem with a fixed context $x$, such that the probability of an action $q$ defined as in Equation~\ref{exp3}, so $p_t(q|x) \propto \exp\pa{- \eta_t \sum_{a=1}^K q_a\iprod{x}{\sum_{s=1}^{t-1} \htheta_{s,a} }}$.  Then, the regret for the fixed $x$ against $\pi^*(x)$ can be written as:
\[\hR_T(x) = \sum_{t=1}^T \EEw{\siprod{z(Q_t(x), x) - z(\pi^*(x),x)}{\htheta_t}}{Q_t(x) \sim p_t(\cdot|x)}.\]
Then it is easy to see that the right-hand side of  (\ref{eq:ghost}) is equal to $\EE{\hR_T(X_0)}$. Thus, we first show a bound on $\hR_T(x) $ that holds almost surely for any $x$ and then take an expectation with respect to $X_0$. 
We control the regret $\hR_T(x)$ by following the general schema of the optimistic mirror descent analysis developed in \citep{RS13nips, NEURIPS2020_15bb63b2}. With this analysis, we get the following bound for any $x\in \X$ : 
\begin{lemma}\label{exp3proof_adaptive} 
	 Assume that $\eta_{t+1}\le \eta_{t}$ for all $t$, let $q_0$ be a uniform distribution over $[K]$ and $\psi(y) = \exp(y) -y -1$. % and let $\htheta_t$ be an unbiased estimate of $\theta_t$. 
	Then, the regret $\hR_T(x)$ of \clinexp almost surely satisfies
%$ \EE{   \sum_{t=1}^T \iprod{ Z_t(X_0) - Z^*(X_0) }{\htheta_t} } \le \EE{ \sum_{t=1}^T  \psi( -\eta_t \iprod{Z_t(X_0)}{\htheta_t - m_t} )  + 2 + \frac{K \log T}{\eta_T}}$.
	\begin{align}\label{regret_x}
		\hR_T(x) & \le \frac{1}{T}\sum_{t=1}^T  \iprod{ z(q_0 - \pi^*(x) , x) }{\ \htheta_t} + \frac{K \log T}{\eta_T} \nonumber\\
		& + \sum_{t=1}^T   \frac{1}{\eta_t} \EEw{  \psi\pa{ -\eta_t \iprod{z(Q_t(x),x)}{\htheta_t - m_t} }}{Q_t(x) \sim p_t(\cdot|x)},
	\end{align}
for $\psi(y) = \exp(y) -y -1$.
\end{lemma}
 We place the derivation of the this bound in the appendix.  The crucial ingredient is to show that the square of the estimated loss can be bounded by the square of the true loss. Using the definition of $\theta_t$, denoting $Var_t = \trace{ \tS_t^{-1} Z_t(X_0)Z_t(X_0)\transpose \tS_t^{-1} Z_t(X_t)Z_t(X_t)\transpose   }$, we get  
\begin{align}\label{var_decomp}
		\EEt{\pa{- \eta_t \iprod{Z_t(X_0)}{\htheta_t- m_t} }^2 } =	\EEt{\eta^2_t \pa{ \ell_t(A_t, X_t) - X_t\transpose m_{t,A_t}}^2 Var_t},
\end{align}
As additional corollary of the concentration result for log-concave random variables, we can show the following relation between matrices $\Sigma_t$ and $\tS_t$:
\begin{equation}\label{cov_matrices_relation}
		\frac{3}{4} \Sigma_t \preceq  \tS_t \preceq \frac{4}{3} \Sigma_t,
\end{equation}
which we prove in Lemma~\ref{logconcconcentration} in the appendix. 
 Then we can show that, almost surely:
\begin{align}\label{quadratic}
	&\EEw{Var_t}{X_0}  = \EEw{\trace{ \tS_t^{-1} Z_t(X_0)Z_t(X_0)\transpose \tS_t^{-1} Z_t(X_t)Z_t(X_t)\transpose   }}{X_0}  \nonumber\\
	& \quad = \trace{ \tS_t^{-1} \Sigma_t \tS_t^{-1} Z_t(X_t)Z_t(X_t)\transpose   }  \le \frac{4}{3}	 \trace{ \tS_t^{-1} \tS_t \tS_t^{-1} Z_t(X_t)Z_t(X_t)\transpose   } \nonumber  \\
	& \quad = \frac{4}{3}	  Z_t(X_t)\transpose   \tS_t^{-1} Z_t(X_t)   \le  Z_t(X_t)\transpose   \Sigma_t^{-1} Z_t(X_t)   \le dK  \gamma^2.
\end{align}
where the first inequality follows from  (\ref{cov_matrices_relation}) and the second inequality is immediate from (\ref{cov_matrices_relation}) and the fact that for symmetric positive definite matrices $A\succeq B$ follows from $B^{-1}\succeq A^{-1}$. The last inequality follows from (\ref{truncation_event}) in the \clinexp. So, from (\ref{var_decomp}) and (\ref{quadratic}), we get 
\begin{align*}
	\EEt{\pa{- \eta_t \iprod{Z_t(X_0)}{\htheta_t- m_t} }^2 } \le  dK  \gamma^2 \EEt{\eta^2_t \pa{ \ell_t(A_t, X_t) - X_t\transpose m_{t,A_t}}^2 }, 
\end{align*}
which, as we stated above, is the key step to prove Theorem~\ref{var_theorem}.
%For the first term of (\ref{regret_x}), we simply get from $-1\le \ell_t \le 1$ and $\htheta_t$ is unbiased:
%\begin{equation}\label{sm_l}
%	\EE{\frac{1}{T}\sum_{t=1}^T  \iprod{ z(q_0 - \pi^*(X_0) , X_0) }{\ \htheta_t} }\le 2.
%\end{equation}
%Before collecting the terms, we show that learning rate $\eta_t$ ---
%defined through the cumulative variance --- can be well approximated by the expected loss of the learner by using an adaptation of Freedman's inequality. This is stated in Lemma~\ref{freedman} of the appendix, where it is shown that with probability at least $1-\delta$ we have for all $t\in[T]$ that $	V_t \le \widehat{V}_t  + 8 \sqrt{\widehat{V}_t  \ln(2T/\delta) + 72\ln(2T/\delta)^2  }  + 88\ln(T/\delta)$. This allows us to show that the learning rate $\eta_t$ is well-concentrated around the value as if it was computed with known $V_{t-1}$. 
%Using this, we get a final bound, using the expression for $\eta_t$ and taking $\delta = 1/T$:
%\begin{align}\label{regret_final}
%	R_T &\le \EE{2dK  \gamma^2 \sum_{t=1}^T \eta_t	 \pa{ \ell_t(A_t, X_t) - X_t\transpose m_{t,A_t}}^2}+ 2 + \EE{\frac{K \log T}{\eta_T}} \nonumber\\
%	&\le 2\sqrt{d} K \gamma^2 \sum_{t=1}^T \frac{V_t - V_{t-1}}{\sqrt{V_t}}  +\tOO(K\sqrt{d V_T})\nonumber	\\
%	&\le 4\sqrt{d} K \gamma^2 \sum_{t=1}^T (\sqrt{V_t} - \sqrt{V_{t-1}})  +\tOO(K\sqrt{d V_T}) %4\sqrt{d} K \gamma^2 \sqrt{V_T} + \tOO(K\sqrt{d V_T}).\nonumber 
%\end{align}
\paragraph{First-order regret bound}
To prove result of Theorem~\ref{first_order}, we show that the bound in the Theorem~\ref{var_theorem} can instantiated to obtain a first-order regret bound with a different choice of the learning rate $\eta_t$. Going along the same lines with regard to the concentration of $\hL_t$ as for $\widehat{V}_t$, by setting $m_t = \bar{0}$ and noticing that then $V_T \le L_T$ we get
\begin{align*}
	R_T &\le 2dK  \gamma^2 \EE{\sum_{t=1}^T \eta_t	  \ell_t(A_t, X_t)^2 } + \tOO(K\sqrt{d V_T})%\le 2dK  \gamma^2 \EE{\sum_{t=1}^T \eta_t	  \ell_t(A_t, X_t) } + \tOO(K\sqrt{d L_T}).\nonumber	
%	&\le 2\sqrt{d} K \gamma^2 \sum_{t=1}^T \frac{L_t - L_{t-1}}{\sqrt{L_t}}  +\tOO(K\sqrt{d L_T}) \le 4\sqrt{d} K \gamma^2 \sum_{t=1}^T (\sqrt{L_t} - \sqrt{L_{t-1}})  +\tOO(K\sqrt{d L_T}) \nonumber \\
\le 4\sqrt{d} K \gamma^2 \sqrt{L_T} + \tOO(K\sqrt{d L_T}).\nonumber 
\end{align*}
Since $R_T = L_t -  L_T^*$, by solving the quadratic inequality with respect to $L_T^*$, we get that $L_T \le L_T^* + \tOO(K\sqrt{d}) $, yielding the final bound. 
%$\EE{(Z^*_t)\transpose(\Sp_t - \Sigma^{-1}_t ) Z_t Z_t\transpose \theta}$

\section{Discussion}\label{sec:discussion}
In conclusion, by applying the approach of \citep{NEURIPS2020_15bb63b2} we have constructed the first scheme achieving $\tilde{O}\left(K\sqrt{dL_{T}^{*}}\right)$ regret with a runtime of %$O\left(K^4 \cdot g_{\Sigma}(T)\right)$, where $g_{\Sigma}(T)$
 $\OO\left((K^5+\mathrm{log}T) \cdot g_{\Sigma}\right)$, where $g_{\Sigma}$
 is the time taken to construct the covariance matrix per round - a potentially large polynomial improvement over the $\OO\left(T^{Kd}\right)$ runtime of MYGA. The application of linear bandit algorithms to the contextual bandit problem constitutes, to the best of our knowledge, a novel approach. In doing so we've found a number of positive aspects, including efficiency, but also the direct applicability of other properties enjoyed by the algorithm such as second order bounds \citep{NEURIPS2020_15bb63b2}.
 
Our approach is based on reducing the linear contextual bandit problem
to a linear bandit problem, as opposed to a multi-armed bandit problem
as in \citep{2020NO}. While the specifics of this reduction heavily
relied on the joint log-concavity of the context distributions and the
exponential-weights posterior over the simplex of actions, we wonder if
such approaches can be successfully applied to achieve other types of
improvements for linear contextual bandits. In particular, it is curious
to what extent other recent advances in the linear bandit problem can be
translated to the linear contextual bandit setting. Note that, while the
truncation step in \ref{alg:linexp3} has an insignificant computational
cost as the condition is satisfied with probability $\OO(1-1/T)$, it can
be removed by paying a $\log(1/\lambda_{min}(\Sigma))$ multiplicative
term in the regret by implementing additional exploration with
probability $1/T$. It is natural to ask whether or not
approaches based on other instantiations of Online Mirror
Descent would also yield first-order bounds, and possibly improve the
dependence on~$K$. The answer is not obvious: for an example of how a
naive application of an instantiation of FTRL fails to achieve a
first-order bound, see Appendix~\ref{sec:exp3_discussion}.

A relevant question pertains to whether or not such an application of algorithms for linear bandits is necessary at all, but standard approaches such as direct adaptation of Exp3, and first-order adaptations thereof such as GREEN \cite{AAGO06} do not seem to give the desired result.%; yielding $R_{T}\leq O((L_{T}^{*})^{2/3})$ in the standard analysis even if an unbiased and sufficiently bounded estimator is possible. 
In addition, thresholding the worst performing arms inevitably biases the loss estimator due to undersampling of those 
arms for which the threshold has been applied, and the resulting additional bias term picked up in the regret scales 
with $1/\lambda_{\mathrm{min}}(\Sigma_{t,a})$, which may be arbitrarily large. Another standard approach of finding an optimistic estimator yielded no fruit during the course of this study due to the lack of the existence of such an 
estimator without saving all previous losses explicitly.

Our algorithm achieves the regret bound $\OO(K\sqrt{dV_T})$, while the
worst case guarantee of \linexp of \cite{2020NO} is $\OO(\sqrt{dK T})$.
This discrepancy is not surprising as the Algorithm~1 of
\cite{NEURIPS2020_15bb63b2} scales as $\OO(n\sqrt{T})$ ($n$ being the
dimension of the action space for the linear bandit), which arises from
the deployment of continuous exponential weights. MYGA achieves the same $\OO(K\sqrt{dL_{T}^{*}})$ bound due to the number of experts needed to cover the joint set of
additive loss parameters. It is worth here emphasising that no known algorithm achieves a better dependence on $K$ than $\OO(K\sqrt{dL_{T}^{*}})$ for the linear adversarial contextual bandit problem. Meanwhile, if the linear bandit is played on
the $n$-simplex, an improvement to $\sqrt{nT}$ is possible. For further discussion of this point, see Section~28.5 of \cite{lattimore_szepesvari_2020}. It is thus still unclear whether or not the extra factor of $\sqrt{K}$ is necessary if one aims for a first-order bound.

An additional point is that while the MYGA algorithm \cite{ABL18} allows for adversarially chosen contexts, the analysis 
of MYGA for our setting relies heavily on the assumption that contexts are drawn i.i.d.~at each iteration. A natural 
question is then whether or not a similar result is achievable in the adversarial context case. It is known that 
achieving sub-linear regret is not possible even for full-information online learning of one-dimensional threshold 
classifiers when both contexts and losses are adversarial \citep{BPS09, SKS16}, which renders sub-linear regret 
similarly impossible to guarantee for the even harder setting that we consider in this paper. However, we do conjecture 
that we could overcome the assumption that the distribution is known or that we can sample from it by employing a more 
elaborate algorithm to estimate the distribution from the data. Indeed, it is not obvious if the distributional 
assumption of a lower bound to the covariance matrix eigenvalues is entirely necessary, since the regret does not depend 
on this.

Lastly, it would be an interesting challenge to see if a high-probability regret bound could be obtained in the form 
stated in the COLT 2017 open problem \citet{AKLLS17} for this setting, but since a high-probability $O(\sqrt{T})$ has 
not yet been proved for the problem here considered, the latter may be more worthy of focus in the short term.

%\redd{TODO: } In the reward setting the  adaptive guarantees for the regret can be expressed with the cumulative reward 
%of the best arm in the multi-armed bandit setting or of the optimal policy for the contextual setting. But in having 
%this quantity small will mean that even the optimal policy doesn't accumulate much reward, so the approach with 
%first-order bound in the loss setting is of much more adaptive nature. 

\bibliography{refs}
\bibliographystyle{apalike}

\newpage
\appendix
\onecolumn
\allowdisplaybreaks
\section{First-order Bound by Reduction to MYGA}\label{myga_proof}

\begin{proof} 

The MYGA algorithm of \citet{ABL18} competes with a class of experts
$E$, where each expert $e \in E$ provides a stochastic prediction
$\xi_t^e \in \Delta_K$ in each round $t$. It provides the following
expected regret bound with respect to the best expert:
\begin{equation}\label{eqn:myga_regret}
		R_{T}= O\left(\sqrt{K\mathrm{log}(|E|+T)L^{*}_{T}}+K\mathrm{log}(|E|+T)\right).
\end{equation}
Losses for the arms can be adversarial, and are assumed to take values
in $[0,1]$.

We will instantiate the experts to cover the parameter space $\{\beta
\in \reals^{K \times d} : \max_a \|\beta_a\| \leq RT\}$ of potentially
optimal parameters for deterministic policies of the form
\eqref{eqn:linear_classifier}, which we know must contain the optimal
policy $\pi_T^*$ with corresponding parameters $\beta^* =
\E[\sum_{t=1}^T \theta_t]$. The covering number for a ball of radius
$RT$ at precision $\epsilon > 0$ is between $\big(\frac{
RT}{\epsilon}\big)^d$ and $\big(\frac{3 RT}{\epsilon}\big)^d$, so by
taking the Cartesian product of this covering with itself $K$ times we
can cover all $\beta$ with $\big(\frac{RT}{\epsilon}\big)^{Kd} \leq |E|
\leq \big(\frac{3 RT}{\epsilon}\big)^{Kd}$ points
$\beta^1,\ldots,\beta^{|E|}$. Let $\ddot\beta \in
\{\beta^1,\ldots,\beta^{|E|}\}$ be the closest point in the covering to
the optimal parameters $\beta^*$. Then its expected approximation error
can be upper bounded as follows:
\begin{align*}
\mathbb{E}\left[\sum_{t=1}^{T}\IP{X_{t}}{\theta_{t,\pi_{\ddot \beta}(X_{t})}}-\IP{X_{t}}{\theta_{t,\pi_{\beta^*}(X_{t})}}\right] &=\mathbb{E}\left[\sum_{t=1}^{T}\IP{X_{0}}{\theta_{t,\pi_{\ddot \beta}(X_{0})}}-\IP{X_{0}}{\theta_{t,\pi_{\beta^*}(X_{0})}}\right] \\ 
		&=\mathbb{E}\left[\IP{X_{0}}{\beta^*_{\pi_{\ddot
                \beta}(X_{0})}}-\IP{X_{0}}{\beta^*_{\pi_{\beta^*}(X_{0})}}\right] \\
                &\leq \mathbb{E}\left[\IP{X_{0}}{\ddot \beta_{\pi_{\ddot
                \beta}(X_{0})}}-\IP{X_{0}}{\beta^*_{\pi_{\beta^*}(X_{0})}}\right]
                + \sigma \epsilon\\
                &= \mathbb{E}\left[\min_a \IP{X_{0}}{\ddot
                \beta_a}-\IP{X_{0}}{\beta^*_{\pi_{\beta^*}(X_{0})}}
                 \right]
                + \sigma \epsilon\\
                &\leq \mathbb{E}\left[\IP{X_{0}}{\ddot
                \beta_{\pi_{\beta^*}(x_{0})}}-\IP{X_{0}}{\beta^*_{\pi_{\beta^*}(x_{0})}}
                 \right]
                + \sigma \epsilon\\
                &\leq \mathbb{E}\left[\IP{X_{0}}{
                \beta^*_{\pi_{\beta^*}(x_{0})}}-\IP{X_{0}}{\beta^*_{\pi_{\beta^*}(x_{0})}}
                 \right]
                + 2\sigma \epsilon\\
                &= 2\sigma \epsilon.
	\end{align*}
        Adding this to \eqref{eqn:myga_regret}, instantiated with $|E| \leq
        \big(\frac{3 RT}{\epsilon}\big)^{Kd}$, and choosing $\epsilon =
        \frac{dK^{2}}{2}$ completes the proof.
\end{proof}

\section{Auxiliary lemmas}

To ensure that step 2 in \clinexp  is defined correctly, we show that the matrix $\Sigma_t$ is full rank:
\begin{lemma} Let the distribution of $X_t$ be such that $\lambdamin ( \EE{X_t X_t\transpose})> 0$. Then, we can show
	\begin{equation}\label{sigmafullrank}
		\lambdamin(\Sigma_{t,a})>0
	\end{equation}
	for any $a\in[K]$, 
	and consequently
	\begin{equation}\label{sigmafullrank2}
		\lambdamin(\Sigma_{t})>0
	\end{equation} 
%	When the sampling of is restricted to the set $\Omega \subset \Delta^K$ such that for any $q\in \Omega$, $q_a \ge 1/T$ for any $a$, we have
%	\begin{equation}\label{sigmafullrank_exploration}
%		\lambdamin(\Sigma_{t})>\frac{\lambdamin(\EE{X_t X_t\transpose})}{T^2}.
%	\end{equation}
\end{lemma}
\begin{proof}
	To show that $\Sigma_{t,a} $ is full rank, it suffices to show
	that there is no $v\in \real^d$ such that $v\transpose
	\Sigma_{t,a} v= 0$. Suppose, to the contrary, that such a $v$ does exist. Then 
	$0 = v\transpose \EEt{Q_{t,a}^2(X_t)X_t
		X_t\transpose  } v
	= \EEt{Q_{t,a}^2(v\transpose X_t)^2}$,
	which implies that $Q_{t,a} v\transpose X_t = 0$ almost surely.
	Since $Q_{t,a}>0$ almost surely, it follows that in fact $v\transpose
	X_t = 0$ almost surely and therefore
	$0 = \EEt{(v\transpose X_t)^2} = v\transpose \EEt{X_t
		X_t\transpose  } v$. But this contradicts our assumption that
	$\lambdamin(\EEt{X_t X_t\transpose  }) > 0$. 
\end{proof}

We will use  a simple
corollary of Freedman’s inequality \cite{F75} that was introduced in Lemma~2 in \cite{BDHKRT08}:
\begin{lemma}\label{freedman}
	Let $Y_1, \dots, Y_t$ be a martingale difference sequence with respect to a filtration $\F_1\subset \dots \subset \F_t$ such that $\EEc{Y_s}{\F_s} = 0$. Suppose that  $Y_s \le b$ holds almost surely. Then with probability at least $1-\delta$ we have $\sum_{s=1}^t Y_s \le 2 \max\{2 \sqrt{\sum_{s=1}^t \EEc{  Y_s^2   }{\F_s}}, b\sqrt{\ln(1/\delta)}  \} \sqrt{\ln(1/\delta)} $. 
\end{lemma}

\section{Proof of Theorem~\ref{var_theorem} }\label{ap:proof}

The proof of Theorem~\ref{var_theorem} proceeds in a sequence of lemmas.
First, we need to show that the distribution of $Z_t(X_{t})$ is log-concave for all $t\in [T]$, 
and after we follow the  analysis of Algorithm 1 of  \cite{NEURIPS2020_15bb63b2},  bounding both components of (\ref{regret_dec}) taking into account the required alterations to incorporate contextual structure
\begin{lemma}\label{logconcavitylemma}
	Suppose $z(q,x)=\sum_{a}q_{a}\varphi(x,a)$ for $\varphi(x,a)=(\bar{0}^{\intercal},\dots,x^{\intercal},\cdots)$ such that $x$ is on the $da$'th co-ordinate and $Q(x) \sim p(\cdot|x)$ for $p(\cdot |x)$ log-concave. If $X \sim p_{X}(\cdot)$ and $p_{X}(\cdot)$ is log-concave and $Z(x)=z(Q_{t}(x),x)$, then $Z(X)$ also follows a log-concave distribution.
\end{lemma}
\begin{proof}
	Assume that $|x^{i}|>0$ for all $i\in[d]$. Set $(z_1, \dots, z_{K-1}, x) = h(q_1 \bar{1}, \dots, q_{K-1}\bar{1}, x)$, where $h: \real^{dK}\to  \real^{dK} $ and $z_i = (\dots, z_i^j, \dots)\transpose$ for each $i \in \{1,\dots,K-1\}$. Thus $z_{i}^{j}=h_i(q_i, x^j)=q_i (x^j )$ and $h_K(x) = x$. The Jacobian of $h^{-1}(z_1, \dots, z_{K-1}, x)$ can be expressed as the block matrix
	
	\[
	J(h^{-1}(z_{1},\dots,z_{K-1},x)) =
	\begin{bmatrix}
		\Lambda_{x} & \Gamma_{z,x} \\
		0_{d\times (K-1)} & \mathrm{Id}_{d\times d}
	\end{bmatrix},
	\]
	where $\Lambda_{x}\in \real^{(K-1)\times (K-1)}$ is diagonal with $(\Lambda_{x})_{ii}=\frac{1}{x^{i} }$ and $\Gamma_{z,x} \in \real^{(K-1)\times d}$ with $(\Gamma_{z,x})_{ij}=-\frac{z^{j}_{i}}{(x^{j} )^{2}}$. Since $J(h^{-1}(z_{1},\dots,z_{K-1},x))$ is upper-triangular, $\det(J(h^{-1}(z_{1},\dots,z_{K-1},x))) = \pa{\prod_{i=1}^{d}\frac{1}{x^i }}^{K-1}$. The joint distribution of $Z_{i}$ and $X$ can thus be written
	\begin{equation*}p_{Z_1, \dots, Z_{K-1}, X}(z_1, \dots, z_{K-1}, x) = p_{Q,X}\pa{h^{-1}(z_1, \dots, z_{K-1}, x)} \pa{\prod_{i=1}^{d}\frac{1}{x^i }}^{K-1}
	\end{equation*}
	with the joint distribution between $Q$ and $X$ of the form
	\begin{equation*}
		p_{Q,X}\pa{h^{-1}(z_1, \dots, z_{K-1}, x)}= \frac{e^{-\eta\iprod{\psi(z,x,\varphi)}{\hTheta_{t-1}}}}{ \int_{q'\in C}  e^{-\eta\iprod{ \sum_{a=1}^{K-1} q_a'\varphi(x,a) }{\hTheta_{t-1}}} dq'}p_{X}(x)
	\end{equation*}
	where $(\psi(z,x,\varphi))_{i}=\sum_{a=1}^{K-1}  \frac{z_a^i}{x^i }\varphi(x,a)^i + \pa{1- \sum_{a=1}^{K-1}  \frac{z_a^i}{x^i }}\varphi(x,K)^i$ has been defined for readability. We can reabsorb the factor $\pa{\prod_{i=1}^{d}\frac{1}{x^i }}^{K-1}$ in the denominator to rewrite the normalization constant as a in terms of the random variable $Z(x)$, and so
	\begin{equation*}
		p_{Z_1, \dots, Z_{K-1}, X}(z_1, \dots, z_{K-1}, x) = \frac{e^{-\eta\iprod{\psi(z,x,\varphi)}{\hTheta_{t-1}}}}{\int_{z'\in Z(x)}  e^{-\eta\iprod{ \psi(z',x,\varphi)}{\hTheta_{t-1}}}dz'}p_{X}(x).
	\end{equation*}
	Define a new function $g:\real^{d\times K} \to \real^{d\times K}$ such that $y = g(z_1, \dots, z_{K-1}, x) = (\dots, g_i(z_1, \dots, z_{K-1}, x), \dots) \transpose$, where for $i \in [1,K-1]$, $g_i( z_1, \dots, z_{K-1}, x) = z_i$ and $g_K(z_1, \dots, z_{K-1}, x) =(\dots, \pa{1 - \frac{1}{x^i }\sum_{a=1}^{K-1} z^i_a} x^i, \dots )$. Then for $i \in \{1,\dots,K-1\}$, $g^{-1}_i(y) = y_i$ and $g^{-1}_K(y) = \sum_{a=1}^K y_a$. The determinant $\mathrm{det}(J(g^{-1}(y)))=1$, so
	\begin{align*}
		p_{Y}(y) &= p_{Z_1, \cdot, Z_{K-1}, X}(g^{-1}(y))\\
		&   = \frac{e^{-\eta\iprod{y}{\hTheta_{t-1}}}}{ \int_{y'\in Y(y)}  e^{-\eta\iprod{y'}{\hTheta_{t-1}}} dy'   }  p_{X}\pa{\sum_{a=1}^K y_a}.
	\end{align*}
	Since both $p_X$ and $\frac{e^{-\eta\iprod{y}{\hTheta_{t-1}}}}{ \int_{y'\in Y(y)}  e^{-\eta\iprod{y'}{\hTheta_{t-1}}} dy'}$ are both log-concave, the lemma follows. 
\end{proof}
Having shown the log-concavity of $Z(X_{t})$, we may safely proceed.

 %The regret may be decomposed as follows.
%\begin{equation}\label{regret_dec}
%	R_t = \EE{ \sum_{t=1}^T \iprod{\tZ_t(X_t) - Z_t(X_t) }{\theta_t}}  + \EE{   \sum_{t=1}^T \iprod{ Z_t(X_t) - Z^*(X_t) }{\theta_t} }.
%\end{equation}
%To ensure that the additional structure acquired by incorporating the context leads to a non-zero distribution with respect to the Lebesgue measure, we must show that $\Sigma_t$ is of full rank: 
We state the analog of Lemma~4 in \cite{NEURIPS2020_15bb63b2} adapted to our setting, leading to a bound on the first term of (\ref{regret_dec}) as well as providing a useful relation between $ \Sigma_t$ and $\tS_t $.
\begin{lemma}\label{logconcconcentration}
	\begin{equation}\label{losses_diff}
		\bigg|\EEt{  \iprod{Z_t(X_t) - \tZ_t(X_t) }{\theta_t}  } \bigg| \le \frac{1}{2 T^2},
	\end{equation}
	and we have
	\begin{equation}\label{matrix_ineq}
		\frac{3}{4} \Sigma_t \preceq  \tS_t \preceq \frac{4}{3} \Sigma_t.  
	\end{equation}
\end{lemma}
\begin{proof}
	From definition of $\tilde p_t$, for any $x \in \X, \theta \in \Theta$, we have
	\begin{align*}\allowdisplaybreaks
		&\EEt{\iprod{\tZ_t(X_t)}{\theta}}\\
		&\qquad	= \frac{1}{\PPt{\norm{
					Z_t(X_t)}_{\Sigma_t^{-1}}^2\le dK\gamma^2}} \int_{\Delta^K} \int_{\X} \iprod{z(q,x)}{\theta} \II{\norm{
				z(q,x)}_{\Sigma_t^{-1}}^2\le dK\gamma^2} p_t(q|x)p(x) dx dq\\
		&\qquad = \frac{1}{1-\delta} \int_{\Delta^K}  \int_{\X} \iprod{z(q,x)}{\theta} \II{\norm{
				z(q,x)}_{\Sigma_t^{-1}}^2\le dK\gamma^2} p_t(q|x) p(x) dx dq \\
		&\qquad = \frac{1}{1-\delta}  \pa{ \EEt{ \iprod{Z_t(X_t)}{\theta} } -  \int_{\Delta^K}  \int_{\X} \iprod{z(q,x)}{\theta} \II{\norm{
					z(q,x)}_{\Sigma_t^{-1}}^2> dK\gamma^2} p_t(q|x)p(x) dx dq },
	\end{align*}
	where $\delta = \PPt{\norm{
			Z_t(X_t)}_{\Sigma_t^{-1}}^2> dK\gamma^2}$. Plugging this into the l.h.s. of (\ref{losses_diff}) yields
	\begin{align*}\allowdisplaybreaks
		&\bigg|\EEt{  \iprod{Z_t(X_t) - \tZ_t(X_t) }{\theta_t}  } \bigg| \\
		&\qquad=  \frac{1}{1-\delta} \bigg|  \delta \EEt{  \iprod{Z_t(X_t) }{\theta_t}  } +   \int_{\Delta^K}  \int_{\X} \iprod{z(q,x)}{\theta} \II{\norm{
				z(q,x)}_{\Sigma_t^{-1}}^2> dK\gamma^2} p_t(q|x)p(x) dx  dq   \bigg| \\
		&\qquad \le  \frac{1}{1-\delta} \pa{  \delta  +   \int_{\Delta^K}  \II{\norm{
					z(q,x)}_{\Sigma_t^{-1}}^2> dK\gamma^2} p_t(q|x)p(x)  dq   } = \frac{2\delta}{1-\delta}.
	\end{align*}
	Since the distribution of $Z_t(X_t)$ is log-concave (Lemma~\ref{logconcavitylemma}), we can apply Lemma~1 of \cite{NEURIPS2020_15bb63b2} to $x = \Sigma_t^{-1/2}Z_t(X_t)$. The assumptions of Lemma~1 of \cite{NEURIPS2020_15bb63b2} hold since we have $\EE{xx\transpose} = I$ and since log-concavity is preserved under linear maps. Using Lemma~1 of \cite{NEURIPS2020_15bb63b2}, we have
	\begin{align*}
		\delta =  \PPt{\norm{
				Z_t(X_t)}_{\Sigma_t^{-1}}^2> dK\gamma^2} \le dK \exp(1-\gamma) \le 3dK \exp(-\gamma) \le \frac{1}{6T^2},
	\end{align*} 
	where the last inequality follows from $\gamma \ge 4 \log (10 dKT )$, which obtains (\ref{losses_diff}). We proceed to showing (\ref{matrix_ineq}). For any $y \in \real^{dK}$, we have
	\begin{align*}
		y\transpose \tS_t y &= \EE{(y\transpose \widetilde Z_t(X_t))^2} = \frac{1}{1-\delta} \EEt{(y\transpose Z_t(X_t))^2 \II{\norm{
					Z_t(X_t)}_{\Sigma_t^{-1}}^2\le  dK\gamma^2} }\\
		&\le \frac{1}{1-\delta} \EEt{(y\transpose Z_t(X_t))^2 }  = \frac{1}{1-\delta}  y\transpose \Sigma_t y. 
	\end{align*}
	Since this holds for all $y \in \real^{dK}$ and $\frac{1}{1-\delta}  \le \frac{4}{3}$, the second inequality in (\ref{matrix_ineq}) holds. Furthermore, we have 
	\begin{align}\label{matrix_diff}\allowdisplaybreaks
		y\transpose \Sigma_t y  -  y\transpose \tS_t y  &=  \EEt{(y\transpose Z_t(X_t))^2 } - \frac{1}{1-\delta} \EEt{(y\transpose Z_t(X_t))^2 \II{\norm{
					Z_t(X_t)}_{\Sigma_t^{-1}}^2\le  dK\gamma^2} } \nonumber\\
		& \le \EEt{(y\transpose Z_t(X_t))^2 \II{\norm{
					Z_t(X_t)}_{\Sigma_t^{-1}}^2 > dK\gamma^2} }\nonumber \\
		&\le y\transpose \Sigma_t y  \EEt{ \norm{Z_t(X_t)}^2_{\Sigma_t^{-1}}  \II{  \norm{Z_t(X_t)}^2_{\Sigma_t^{-1}} > dK \gamma^2 } },
	\end{align}
	where the last inequality follows from Cauchy-Schwartz:
	\[ (y\transpose Z_t(X_t))^2 =  \pa{  \iprod{ \Sigma_t^{1/2}y  }{\Sigma_t^{-1/2} x} }^2 \le \norm{\Sigma_t^{1/2}y }_2^2 \cdot \norm{\Sigma_t^{-1/2}x }_2^2 = y\transpose \Sigma_t y \norm{x}^2_{\Sigma_t^{-1}}.  \]
	The right-hand side of (\ref{matrix_diff}) can be bounded using Lemma~1 of \cite{NEURIPS2020_15bb63b2} as follows:
	\begin{align}\label{matrixnormbound}\allowdisplaybreaks
		&\EEt{ \norm{Z_t(X_t) }^2_{ \Sigma_t^{-1}  }  \II{  \norm{Z_t(X_t)}^2_{\Sigma_t^{-1}} > dK \gamma^2 }  }\nonumber\\
		&\quad \le \sum_{n=1}^{\infty}(n+1)^2dK \gamma^2\PPt{ n^2dK\gamma^2 \le \norm{Z_t(X_t)}^2_{\Sigma_t^{-1}}  \le (n+1)^2dK\gamma^2  }\nonumber\\
		&\quad \le \sum_{n=1}^{\infty} (n+1)^2(dK )^2 \gamma^2 \exp(1 - n\gamma)\nonumber\\
		&\quad \le (dK )^2 \gamma^2 \sum_{n=1}^{\infty} \exp(2+n - n\gamma) = (dK )^2 \gamma^2 \frac{\exp(3-\gamma)}{1 - \exp(1-\gamma)} \le \frac{1}{4}.
	\end{align}
	Combining (\ref{matrixnormbound}) and (\ref{matrix_diff}) we get the first inequality of (\ref{matrix_ineq}).
\end{proof}
%To control the second term of regret decomposition (\ref{regret_dec}), we use an idea of the reduction of the contextual bandit problem to a combination of auxiliary online learning problems that are defined separately for each context, that was proposed in \cite{2020NO} and stated there as Lemma 3. \todo{add that it is a reduction to the linear bandit problem}
\begin{lemma}\label{ghost1} Let $\pi^*$ be any fixed stochastic policy and let $X_0\sim\DD$ be a sample from the context distribution independent from $\F_T$. Suppose that $p_t \in \F_{t-1}$, such that $p_t(\cdot|x)$ is a probability density with respect to Lebesgue measure with support $\Delta^K$ and let $Q_t(x)\sim p_t(\cdot|x)$. Then,
	%	\begin{align}\label{eq:ghost}
	%		&\EEt{  \iprod{Z_t(X_t) -  Z^*(X_t)}{\theta_t}  } = \EEt{   \iprod{Z_t(X_0) -  Z^*(X_0)}{\htheta_t}}.
	%	\end{align}
	\begin{align}
		&\EE{ \sum_{t=1}^T \iprod{z(Q_t(X_t), X_t) -  z(\pi^*(X_t), X_t)}{\theta_t}  }\nonumber  = \EE{ \sum_{t=1}^T  \iprod{z(Q_t(X_0), X_0) -  z(\pi^*(X_0)), X_0}{\htheta_t}}.
	\end{align}
\end{lemma}
\begin{proof}
	For any $t$, we have
	\begin{align*}
		&\EEt{   \iprod{Z_t(X_0) -  Z^*(X_0)}{\htheta_t}} =  \EEt{  \EEcct{ \iprod{Z_t(X_0) -  Z^*(X_0)}{\htheta_t}}{X_0}  }\\
		&\quad = \EEt{  \EEcct{ \iprod{Z_t(X_0) -  Z^*(X_0)}{\theta_t}}{X_0}  } = \EEt{  \iprod{Z_t(X_t) -  Z^*(X_t)}{\theta_t}  }. 
	\end{align*}
\end{proof} 
%\redd{add to the analysis, now we need to study as an instance on linear bandits} For any $x \in \X$, we can introduce the regret at the context $x$ as
% \[\hR_T(x) = \sum_{t=1}^T \EEw{\siprod{z(Q_t(x), x) - z(\pi^*(x),x)}{\htheta_t}}{Q_t(x) \sim p_t(\cdot|x)}.\]
Then, we prove the almost sure regret bound for  any $x$ and then take an expectation over $X_0$. We further proceed with an adaptation of the analysis of the continuous exponential weights algorithm, which was stated in \cite{NEURIPS2020_15bb63b2} as Lemma~16,  but we include it here for the clarity.  
Let $\psi(y) = \exp(y) -y -1$. For any $x\in \X$, we show the following : 
\begin{lemma}\label{exp3proof_adaptive_ap} 
	Assume that $\eta_{t+1}\le \eta_{t}$ for all $t$, let $q_0$ be a uniform distribution over $[K]$ and $\psi(y) = \exp(y) -y -1$. % and let $\htheta_t$ be an unbiased estimate of $\theta_t$. 
	Then, the regret $\hR_T(x)$ for any $x \in \X$ of \clinexp almost surely satisfies
	%$ \EE{   \sum_{t=1}^T \iprod{ Z_t(X_0) - Z^*(X_0) }{\htheta_t} } \le \EE{ \sum_{t=1}^T  \psi( -\eta_t \iprod{Z_t(X_0)}{\htheta_t - m_t} )  + 2 + \frac{K \log T}{\eta_T}}$.
	\begin{align*}\label{regret_x}
		&\hR_T(x)  \le \frac{1}{T}\sum_{t=1}^T  \iprod{ z(q_0 - \pi^*(x) , x) }{\ \htheta_t} + \frac{K \log T}{\eta_T} + \sum_{t=1}^T   \frac{1}{\eta_t} \EEw{  \psi\pa{ -\eta_t \iprod{z(Q_t(x),x)}{\htheta_t - m_t} }}{Q_t(x) \sim p_t(\cdot|x)}.
	\end{align*}
\end{lemma}
\begin{proof}
	Note that we can write $\hR_T(x)$ as 
	\[\hR_T(x) = \sum_{t=1}^T \pa{\int_{\Delta^K} p_t(q|x) \iprod{z(q,x)}{\htheta_t} dq -  \iprod{ z(\pi^*(x), x) }{\sum_{t=1}^T \htheta_t} }. \]
	Define $W_t(x) = \int_{\Delta^K} w_t(q|x) dq$, $u_t(q|x) = \exp\pa{ - \eta_t \sum_a q_a \iprod{x}{\sum_{s=1}^{t} \htheta_{s,a}  }  }$, $U_t(x) = \int_{\Delta^K} u_t(q|x)dq$ and $v_t(q|x) = \exp\pa{ - \eta_{t+1} \sum_a q_a \iprod{x}{\sum_{s=1}^{t} \htheta_{s,a}  }  }$, $V_t(x) = \int_{\Delta^K} v_t(q|x)dq$. We have
	\begin{align*}\allowdisplaybreaks
		U_{t} (x) &= \int_{\Delta^K} w_{t}(q|x) \exp\pa{ -\eta_{t} \iprod{z(q, x)}{\htheta_t - m_t}} dq = W_t(x) \int_{\Delta^K} p_t(q|x)\exp\pa{ -\eta_t \iprod{z(q, x)}{\htheta_t - m_t }} dq\\
		&\quad = W_t(x) \int_{\Delta^K} p_t(q|x) \pa{ 1 - \eta_t\iprod{z(q,x)}{\htheta_t - m_t} + \psi( -\eta_t \iprod{z(q,x)}{\htheta_t - m_t} ) } dq. 
	\end{align*} 
	Taking the logarithm of both sides, we get
	\begin{align}\label{eq11}\allowdisplaybreaks
		\log(U_{t}(x) ) &= \log( W_t(x)) + \log \pa{ \int_{\Delta^K} p_t(q|x) \pa{ 1 - \eta_t\iprod{z(q,x)}{\htheta_t - m_t } + \psi( -\eta_t \iprod{z(q,x)}{\htheta_t- m_t} ) }  dq} \nonumber\\
		&\quad \le \log( W_t(x)) +  \int_{\Delta^K} p_t(q|x) \pa{- \eta_t\iprod{z(q,x)}{\htheta_t- m_t}  + \psi( -\eta_t \iprod{z(q,x)}{\htheta_t- m_t} )  }dq ,
	\end{align} 
	where we used the inequality $\log(1+x)\le x$ for $x>-1$. 
	\begin{align}\label{eq33}
		V_{t-1}(x) &= \int_{\Delta^K} w_t(q|x)\exp\pa{ \eta_t \sum_a q_a \iprod{x}{m_{t,a}}} dq = W_t(x) \int_{\Delta^K} p_t(q|x)\exp\pa{ \eta_t \sum_a q_a \iprod{x}{m_{t,a}}} dq\nonumber\\
		& \ge W_t(x)\exp\pa{ \eta_t  \int_{\Delta^K} p_t(q|x) \sum_a q_a \iprod{x}{m_{t,a}} dq}, 
	\end{align}
	using Jensen's inequality. It holds that 
	\begin{align*}
		\int_{\Delta^K} p_t(q|x) \sum_a q_a \iprod{x}{m_{t,a}} dq \le \frac{1}{\eta_t} \log \frac{ V_{t-1}(x)   }{W_t(x)}.
	\end{align*}
	Then, we get 
	\begin{align*}
	 \sum_{t=1}^T \int_{\Delta^K} p_t(q|x) \iprod{z(q,x)}{\htheta_t} dq\le \sum_{t=1}^T \frac{1}{\eta_t} \pa{ \log \frac{ V_{t-1}(x)}{U_t(x)}  +  \int_{\Delta^K} p_t(q|x)  \psi( -\eta_t \iprod{z(q,x)}{\htheta_t- m_t} )  dq }.
	\end{align*}
	Noting that $V_0 = U_0$, we have
	\begin{align*}\allowdisplaybreaks
		\sum_{t=1}^T \frac{1}{\eta_t}\log \frac{ V_{t-1}(x)}{U_t(x)} &= \sum_{t=1}^T \frac{1}{\eta_t} \pa{ \log \frac{ V_{t-1}(x)}{V_0} -   \log \frac{ U_{t}(x)}{U_0}   } \\
		& = 	\sum_{t=1}^{T-1} \pa{ \frac{1}{\eta_{t+1}} \log  \frac{ V_{t}(x)}{V_0}  -   \frac{1}{\eta_{t}} \log  \frac{ U_{t}(x)}{U_0}}      - \frac{1}{\eta_{T}} \log  \frac{ U_{T}(x)}{U_0} 
	\end{align*}
	To bound the first term, we use that $\eta_{t+1}\le \eta_t$ and an additional application of Jensen's inequality:
	\begin{align*}\allowdisplaybreaks
		\frac{1}{\eta_{t+1}} \log  \frac{ V_{t}(x)}{V_0}  &= 	\frac{1}{\eta_{t+1}} \log \EE{ \exp\pa{ - \eta_{t+1} \siprod{\sum_{s=1}^t \htheta_{s} }{ z(Q_t, x)  }  }}\\
		& =	\frac{1}{\eta_{t+1}} \log \EE{ \exp\pa{ - \eta_{t} \siprod{\sum_{s=1}^t \htheta_{s} }{ z(Q_t, x)  }  }^{\frac{\eta_{t+1}}{\eta_t}}}\\
		& \le \frac{1}{\eta_{t}} \log \EE{ \exp\pa{ - \eta_{t} \siprod{\sum_{s=1}^t \htheta_{s} }{ z(Q_t, x)  }  }} = \frac{1}{\eta_{t}} \log \frac{U_t(x)}{U_0},
	\end{align*}
	%Note that $  \int_{\Omega} p_t(q|x) \iprod{z(q,x)}{\htheta_t}  = \EEt{ \iprod{Z_t(x)}{\htheta_t}  }$. 
	Set $Q_{\pi^*(x)} := \{ (1-\frac{1}{T})\pi^*(x) + \frac{1}{T}q| q\in \Delta^K  \} $, and denote $q_0$ as the uniform distribution over $K$ arms. We then have 
	\begin{align*}\allowdisplaybreaks
		U_{T}(x) &\ge \int_{Q_{\pi^*(x)} } \exp\pa{-\eta_T \iprod{ z(q,x)}{\sum_{t=1}^T \htheta_t} }dq\\
		&\qquad  = T^{-K}  \int_{\Delta^K } \exp\pa{  -\eta_T \iprod{ z((1-\frac{1}{T})\pi^*(x) + \frac{1}{T}q, x) }{\sum_{t=1}^T \htheta_t}  } dq\\
		&\qquad \ge  T^{-K}  U_0(x) \exp\pa{  -\eta_T \iprod{ z((1-\frac{1}{T})\pi^*(x) + \frac{1}{T} q_0, x) }{\sum_{t=1}^T \htheta_t}  },
	\end{align*}
	where the first inequality constitutes a change of variables and the second follows from Jensen's bound. After rearranging and taking the logarithm, we get
	\begin{align*}\allowdisplaybreaks
		- \frac{1}{\eta_T} \log \frac{U_{T}(x)}{U_0(x)} &\le \sum_{t=1}^T \iprod{ z((1-\frac{1}{T})\pi^*(x) + \frac{1}{T} q_0, x) }{ \htheta_t} + \frac{K \log T}{\eta_T}\\
		& = \sum_{t=1}^T \iprod{ z(\pi^*(x), x) }{\sum_{t=1}^T \htheta_t} +  \frac{1}{T}\sum_{t=1}^T  \iprod{ z(q_0 - \pi^*(x) , x) }{\ \htheta_t} + \frac{K \log T}{\eta_T}. 
	\end{align*}
	Combining everything together, we get
	\begin{align*}\allowdisplaybreaks
 \sum_{t=1}^T \pa{\int_{\Delta^K} p_t(q|x) \iprod{z(q,x)}{\htheta_t} dq -  \iprod{ z(\pi^*(x), x) }{\sum_{t=1}^T \htheta_t} } &\le \sum_{t=1}^T   \frac{1}{\eta_t}  \int_{\Delta^K} p_t(q|x)   \psi( -\eta_t \iprod{z(q,x)}{\htheta_t - m_t} )dq \\
 & + \frac{1}{T}\sum_{t=1}^T  \iprod{ z(q_0 - \pi^*(x) , x) }{\ \htheta_t} + \frac{K \log T}{\eta_T}.
	\end{align*}
\end{proof}

%\input{appendixB}

%We state the analog of Lemma 4 in \cite{NEURIPS2020_15bb63b2} adapted to our setting, leading to a bound on the first term of (\ref{regret_dec}) as well as providing a useful relation between $ \Sigma_t$ and $\tS_t $.

From Lemma~\ref{ghost} and Lemma~\ref{exp3proof_adaptive}, we get a bound on  the second term of (\ref{regret_dec}):

\begin{align}\label{expectation_decomp}\allowdisplaybreaks
	\EE{ \sum_{t=1}^T \iprod{Z_t(X_t) -  Z^*(X_t)}{\theta_t}  } &= \EE{  \sum_{t=1}^T \iprod{Z_t(X_0) -  Z^*(X_0)}{\htheta_t}} \nonumber \\
	&\le \EE{ \sum_{t=1}^T     \frac{1}{\eta_t}     \psi\pa{ -\eta_t \iprod{Z_t(X_0)}{\htheta_t - m_t} }
		 + \frac{1}{T}\sum_{t=1}^T  \iprod{ z(q_0 - \pi^*(X_0) , X_0) }{\ \htheta_t} + \frac{K \log T}{\eta_T} }.
\end{align} 
We first  find a bound on the first term using Lemma 6 from \cite{NEURIPS2020_15bb63b2}. To satisfy the assumptions of Lemma 6 from \cite{NEURIPS2020_15bb63b2}, we need to show that $\EEt{\pa{- \eta_t \iprod{Z_t(X_0)}{\htheta_t - m_t} }^2 }  \le \frac{1}{100}$:
\begin{align}\label{quadratic}\allowdisplaybreaks
	\EEt{\pa{- \eta_t \iprod{Z_t(X_0)}{\htheta_t- m_t} }^2 } & =	\EEt{\eta^2_t \pa{ \ell_t(A_t, X_t) - X_t\transpose m_{t,A_t}}^2 \trace{ \tS_t^{-1} Z_t(X_0)Z_t(X_0)\transpose \tS_t^{-1} Z_t(X_t)Z_t(X_t)\transpose   }} \nonumber \\
	& = \eta^2_t	\EEt{ \pa{ \ell_t(A_t, X_t) - X_t\transpose m_{t,A_t}}^2 \trace{ \tS_t^{-1} \Sigma_t \tS_t^{-1} Z_t(X_t)Z_t(X_t)\transpose   }} \nonumber  \\
	& \le \eta^2_t \frac{4}{3}	\EEt{\pa{ \ell_t(A_t, X_t) - X_t\transpose m_{t,A_t}}^2 \trace{ \tS_t^{-1} \tS_t \tS_t^{-1} Z_t(X_t)Z_t(X_t)\transpose   }} \nonumber  \\
	& = \eta^2_t \frac{4}{3}	\EEt{ \pa{ \ell_t(A_t, X_t) - X_t\transpose m_{t,A_t}}^2 Z_t(X_t)\transpose   \tS_t^{-1} Z_t(X_t) }\nonumber \\
	& \le  \eta^2_t	\EEt{ \pa{ \ell_t(A_t, X_t) - X_t\transpose m_{t,A_t}}^2  Z_t(X_t)\transpose   \Sigma_t^{-1} Z_t(X_t) }\nonumber  \\ 
	& \le dK \eta^2_t \gamma^2	\EEt{ \pa{ \ell_t(A_t, X_t) - X_t\transpose m_{t,A_t}}^2} \\
	& \le \frac{1}{100},
\end{align}
where the first inequality follows from $\ell_t \le 1$ and (\ref{matrix_ineq}), the second is immediate from (\ref{matrix_ineq}) and the fact that for symmetric positive definite matrices $A\succeq B$ follows from $B^{-1}\succeq A^{-1}$. The third inequality follows from the truncation in the algorithm and the last is immediate from plugging in the definition of $\eta_t$.    
So, by applying Lemma 6 from \cite{NEURIPS2020_15bb63b2} and (\ref{quadratic}), we get:
\begin{align}\label{psi_bound}
	\frac{1}{\eta}\EE{  \psi\pa{ -\eta \iprod{Z_t(X_0)}{\htheta_t - m_t} } } \le  \frac{2}{\eta}	\EE{\pa{- \eta \iprod{Z_t(X_0)}{\htheta_t-m_t} }^2 } \le 2dK \eta \gamma^2 	\EEt{ \pa{ \ell_t(A_t, X_t) - X_t\transpose m_{t,A_t}}^2}.
\end{align}
For the second term of (\ref{expectation_decomp}), we simply get from $-1\le \ell_t \le 1$ and $\htheta_t$ is unbiased:
\begin{equation}\label{sm_l}
	\EE{\frac{1}{T}\sum_{t=1}^T  \iprod{ z(q_0 - \pi^*(X_0) , X_0) }{\ \htheta_t} }\le 2.
\end{equation}
 The expression that we use for the learning rate is the following:
\[
	\eta_t  =   (100dK\gamma^2 + d(\widehat{V}_{t-1} + 1  +
        G_t)))^{-1/2},
\]
where $G_t =  8 \sqrt{2\widehat{V}_{t-1}  \ln T + 144\ln^2 T  }  + 176\ln T$.
 %We show that learning rate $\eta_t$  that is defined through the cumulative variance is finely approximated by the expected loss ob the learner by using an adaptation of Freedman's inequality that was stated in Lemma~\ref{freedman}.
We show that with probability at least $1-\delta$ the following holds for all $t\in[T]$:
\begin{equation}\label{var_conc}
	V_t \le \widehat{V}_t  + 8 \sqrt{\widehat{V}_t  \ln(2T/\delta) + 72\ln(2T/\delta)^2  }  + 88\ln(T/\delta)
\end{equation}
	Let $Y_s =  \EEs{ \pa{ \ell_t(A_t, X_t) - X_t\transpose m_{t,A_t}}^2  } - \pa{ \ell_t(A_t, X_t) - X_t\transpose m_{t,A_t}}^2$. Then, $	Y_s \le 4$ almost surely, since $\ell_t(A_t, X_t) - X_t\transpose m_{t,A_t} \le 2$.
Similarly we bound the second moment of $Y_s$, using Jensen's inequality: 
\begin{align*}\allowdisplaybreaks
	\EEs{Y_s^2} &= \EEs{ \pa{ \EEs{ \pa{ \ell_t(A_t, X_t) - X_t\transpose m_{t,A_t}}^2  } - \pa{ \ell_t(A_t, X_t) - X_t\transpose m_{t,A_t}}^2  }^2}\\
	& \le 2  \EEs{  \pa{ \ell_t(A_t, X_t) - X_t\transpose m_{t,A_t}}^2   }^2 + 2 \EEs{  \pa{ \ell_t(A_t, X_t) - X_t\transpose m_{t,A_t}}^4}\\
	& \le 16\EEs{  \pa{ \ell_t(A_t, X_t) - X_t\transpose m_{t,A_t}}^2   }.
\end{align*}
By Lemma~\ref{freedman},  the following  holds for some $\delta' \in (0,1)$: 
\begin{align}\label{ineq_conc}
	V_t \le \widehat{V}_t + 8 \max \biggl\{ 2 \sqrt{V_t}, \sqrt{\ln(1/\delta')}   \biggr\}\sqrt{\ln(1/\delta')}
\end{align}
Note that this  inequality i can be rearranged as 
\begin{align*}
	V_t \le \widehat{V}_t  + 8 \sqrt{\widehat{V}_t  \ln(1/\delta') + 72\ln(1/\delta')^2  }  + 88\ln(1/\delta'). 
\end{align*}
Then, taking a union bound over $t \in [T]$ and taking $\delta = \delta'/T$, we get that (\ref{var_conc}) holds for all $t \in [T]$.
 Let $\mathcal{E}_T$ be an event that for all $t\in[1,T]$, (\ref{var_conc}) holds with $\delta = 1/T$. From (\ref{psi_bound}), (\ref{sm_l}), and the choice of $\eta_t$, we get:
\begin{align}\label{regret_final}\allowdisplaybreaks
	R_T &\le \EE{2dK  \gamma^2 \sum_{t=1}^T \eta_t	 \pa{ \ell_t(A_t, X_t) - X_t\transpose m_{t,A_t}}^2 + 2 + \frac{K \log T}{\eta_T}} \\
	&=  2dK  \gamma^2 \EE{\sum_{t=1}^T \eta_t \pa{ \ell_t(A_t, X_t) - X_t\transpose m_{t,A_t}}^2 \II{\mathcal{E}_T} }\nonumber \\
	&+ 2dK  \gamma^2 \EE{ \sum_{t=1}^T \eta_t	 \pa{ \ell_t(A_t, X_t) - X_t\transpose m_{t,A_t}}^2  \II{\overline{\mathcal{E}}_T}}  +  2 + \EE{\frac{K \log T}{\eta_T}}  \nonumber \\
	&\le 2\sqrt{d} K \gamma^2 \sum_{t=1}^T \frac{V_t - V_{t-1}}{\sqrt{V_t}}  + \frac{1}{T} 2\sqrt{d}K  \gamma^2  T + 2 + \frac{K \log T}{\eta'_T}    \nonumber 
	\\
	&= 2\sqrt{d}  K\gamma^2 \sum_{t=1}^T \frac{(\sqrt{V_t} - \sqrt{V_{t-1}})(\sqrt{V_t} + \sqrt{V_{t-1}})}{\sqrt{V_t}}   +  2\sqrt{d}K  \gamma^2  + 2 + \frac{K \log T}{\eta'_T} \nonumber \\
	&\le 4\sqrt{d} K \gamma^2 \sum_{t=1}^T (\sqrt{V_t} - \sqrt{V_{t-1}}) +  2\sqrt{d}K  \gamma^2   + 2 + \frac{K \log T}{\eta'_T} \nonumber \\
	&\le 4\sqrt{d} K \gamma^2 \sqrt{V_T} + 2\sqrt{d}K  \gamma^2  + 2 + \frac{K \log T}{\eta'_T}.\nonumber 
\end{align}
which implies the  result of  Theorem~\ref{var_theorem}. In the equation above,  $\eta'_T  =  (100dK\gamma^2 + d(V_{t-1} + 1  + G'_t)))^{-1/2}$ and  $ G_t' = 8 \sqrt{2V_{t-1}  \ln T + 144\ln^2 T  }  + 176\ln T$. In line 4 we used that $\EE{1/\eta_T} \le \EE{1/\eta'_T}$ by Jensen's inequality to show that
\[ \EE{ \frac{1}{\eta_T} } = \EE{ (100dK\gamma^2 + dK(\widehat{V}_{t-1} + 1  + G_t))^{1/2} } \le (100dK\gamma^2 + d(V_{t-1} + 1  + G'_t))^{1/2} =  \frac{1}{\eta_T'}.  \]
\qed

\paragraph{Proof of Theorem~\ref{first_order}}
As it was done in the proof of Theorem~\ref{var_theorem}  we control the deviation of the learning rate 
\[
	\eta_t  =   (100dK\gamma^2 + d(\widehat{L}_{t-1} + 1  + H_t)))^{-1/2},
\]
where $H_t$ is as defined in \eqref{learn_rate_1st}.
Using  Lemma~\ref{freedman},
we show that with probability at least $1-\delta$ the following holds for all $t\in[T]$:
\begin{equation}\label{loss_conc}
	L_t \le \hL_t +  8 \sqrt{\hL_t \ln(1/\delta) + 20\ln(2T/\delta)^2  }  + 36\ln(2T/\delta)
\end{equation}
$D_s =  \EEs{ \siprod{X_s}{\theta_{s,A_s}}  } - \siprod{X_{s}}{ \theta_{s,A_s} }$. Then, $	D_s \le 2$ almost surely and by Jensen's inequality 
	\begin{align*}
	\EEs{D_s^2}=\EEs{ \pa{  \EEs{ \siprod{X_s}{\theta_{s,A_s}}  } - \siprod{X_{s}}{ \theta_{s,A_s} }  }^2} \le 2  \EEs{ \siprod{X_s}{\theta_{s,A_s}}  }^2 + 2 \EEs{ \pa{  \siprod{X_{s}}{ \theta_{s,A_s} }  }^2} \le 4\EEt{\ell_t(X_t, A_t)}.
\end{align*}
By Lemma~\ref{freedman},  the following  holds for some $\delta' \in (0,1)$: 
\begin{align}\label{ineq_conc}
	L_t \le \hL_t + 4 \max \biggl\{ 2 \sqrt{L_t}, \sqrt{\ln(1/\delta')}   \biggr\}\sqrt{\ln(1/\delta')}
\end{align}
which can be rearranged as 
\begin{align*}
		L_t \le \hL_t +  8 \sqrt{\hL_t \ln(1/\delta') + 20\ln(1/\delta')^2  }  + 36\ln(1/\delta').
\end{align*}
Then, taking a union bound over $t \in [T]$ and taking $\delta = \delta'/T$, we get that (\ref{loss_conc}) holds for all $t \in [T]$.
 Let $\mathcal{E}_T$ be an event that for all $t\in[1,T]$, (\ref{loss_conc}) holds with $\delta = 1/T$. From (\ref{psi_bound}), (\ref{sm_l}),  the choice of $\eta_t$, $m_t = \bar{0}$ and since $0\le \ell_{t} \le 1$, we get:

 \begin{align*}
 	R_T &\le \EE{2dK  \gamma^2 \sum_{t=1}^T \eta_t	  \ell_t^2(A_t, X_t) + 2 + \frac{K \log T}{\eta_T}} \le  \EE{2dK  \gamma^2 \sum_{t=1}^T \eta_t	  \ell_t(A_t, X_t) + 2 + \frac{K \log T}{\eta_T}} \\
 	&=  2dK  \gamma^2 \EE{\sum_{t=1}^T \eta_t \ell_t(A_t, X_t)  \II{\mathcal{E}_T} } + 2dK  \gamma^2 \EE{ \sum_{t=1}^T \eta_t	 \ell_t^2(A_t, X_t)  \II{\overline{\mathcal{E}}_T}}  +  2 + \EE{\frac{K \log T}{\eta_T}}  \nonumber \\
 	&\le 2\sqrt{d}K  \gamma^2 \sum_{t=1}^T \frac{L_t - L_{t-1}}{\sqrt{L_t}}  + \frac{1}{T} 2\sqrt{d}K  \gamma^2  T + 2 + \frac{K \log T}{\eta'_T}    \nonumber 
 	\\
 	&= 2\sqrt{d}K  \gamma^2 \sum_{t=1}^T \frac{(\sqrt{L_t} - \sqrt{L_{t-1}})(\sqrt{L_t} + \sqrt{L_{t-1}})}{\sqrt{L_t}}   +  2\sqrt{dK}  \gamma^2  + 2 + \frac{K \log T}{\eta'_T} \nonumber \\
 	&\le 4\sqrt{d} K \gamma^2 \sum_{t=1}^T (\sqrt{L_t} - \sqrt{L_{t-1}}) +  2\sqrt{d}K  \gamma^2   + 2 + \frac{K \log T}{\eta'_T} \nonumber \\
 	&\le 4\sqrt{d} K \gamma^2 \sqrt{L_T} + 2\sqrt{d}K  \gamma^2  + 2 + \frac{K \log T}{\eta'_T},\nonumber 
 \end{align*}
where in the equation above,   $\eta'_t  =  (100dK\gamma^2 + d(L_{t-1} + 1  + H'_{t-1})))^{-1/2}$ and $ H_t' = 8 \sqrt{2L_{t-1} \ln T + 40\ln T  }  + 72\ln T$.
By solving the quadratic equation over $L_T^*$, we obtain the statement of the theorem. 

\qed

\section{On the difference between \linexp and \clinexp}\label{sec:exp3_discussion}

Consider the \linexp algorithm of \cite{2020NO}, that draws actions after observing the context $X_t$ with probability $$ \pi_{t}\left(a\middle|X_t\right)= (1-\gamma) \frac{w_{t}(X_t,a)}{\sum_{a'} 
	w_{t}(X_t,a')} + \frac{\gamma}{K},$$ where
$w_{t}(X_t,a) = \exp\pa{- \eta  \sum_{s=0}^{t-1} \langle X_t, \htheta_{s,a}\rangle}$ and using the estimator $$\ttheta^*_{t,a} = \II{A_t = a} S_{t,a}^{-1} X_t \iprod{X_t}{\theta_{t,a}},$$ where $S_{t,a} = \EEt{\pi_t(a|X_t) X_tX_t\transpose}$. 
Since \linexp  uses implicit exploration with probability $\gamma$,  $\lambda_{min}(S_{t,a}) \ge \lambda_{min}(\Sigma) \frac{\gamma}{K}$. But then, setting $\gamma = 0$, $S_{t,a}$ is still invertible as no actions have $\pi_{t}\left(a\middle|X_t\right)=0$. But still, the smallest eigenvalue $\lambdamin(S_{t,a})$ can be arbitrary small. Then, the analysis of the variance term in \linexp looks as:
\begin{align*}
	\EEt{\sum_{a=1}^K\pi_t(a|X_0) \siprod{X_0}{\htheta_{t,a}}^2 } &= \EEt{\sum_{a=1}^K \pi_t(a|X_0) \pa{X_0\transpose S_{t,a}^{-1} X_t X_t\transpose \theta_{t,a}\II{A_t = a} }^2} \\
	&\quad = \EE{ \ell_{t}(X_t, A_t)^2 \trace{ \pi_t(a|X_0)X_0X_0\transpose S_{t,a}^{-1}X_t X_t\transpose S_{t,a}^{-1}}}.
\end{align*}
We can define $Var'_t$ for \linexp in direct analogy to $Var_t$ for \clinexp above, which gives (almost surely):
\begin{align*}
	\EEw{Var_t'}{X_0} &= \EEw{\trace{ \pi_t(a|X_0)X_0X_0\transpose S_{t,a}^{-1}X_t X_t\transpose S_{t,a}^{-1}}}{X_0}\\
	&\quad =  \EEw{\trace{ \Sigma_{t,a} S_{t,a}^{-1}X_t X_t\transpose S_{t,a}^{-1}}}{X_0}  =  X_t\transpose S_{t,a}^{-1} X_t,
\end{align*} 
which can be arbitrary large. 

Meanwhile, the smallest eigenvalue $\lambdamin(\Sigma_{t,a})$ can be arbitrary small too. But, as we showed above in the analysis of \clinexp, $Var_t$ is bounded by $dK  \gamma^2$ because of the log-concavity of $Z_t(X_t)$ and step (\ref{truncation_event}) of \clinexp.

\end{document}